\documentclass[review]{elsarticle}

\usepackage{lineno,hyperref}
\usepackage{amsmath}
\usepackage{amssymb}
\usepackage{amsthm}
\usepackage{algorithmic}
\usepackage{algorithm}
\usepackage{subfigure}
\modulolinenumbers[5]
\newtheorem{thm}{Theorem}

\newtheorem{cor}[thm]{Corollary}

\journal{Journal of \LaTeX\ Templates}

\begin{document}

\begin{frontmatter}

\title{Fast Kernel $k$-means Clustering Using Incomplete Cholesky Factorization}


\author[mymainaddress,mysecondaryaddress]{Li Chen}

\author[mysecondaryaddress]{Shuisheng Zhou\corref{mycorrespondingauthor}}
\cortext[mycorrespondingauthor]{Corresponding author}
\ead{sszhou@mail.xidian.edu.cn}

\author[mysecondaryaddress]{Jiajun Ma}

\address[mymainaddress]{College of Physical Education (Main Campus), Zhengzhou University, 100 Science Avenue, Zhengzhou, China}
\address[mysecondaryaddress]{School of Mathematics and Statistics, Xidian University, 266 Xinglong Section, Xifeng Road, Xi'an, China}

\begin{abstract}
  Kernel-based clustering algorithm can identify and capture the non-linear structure in datasets, and thereby it can achieve better performance than linear clustering. However, computing and storing the entire kernel matrix occupy so large memory that it is difficult for kernel-based clustering to deal with large-scale datasets. In this paper, we employ incomplete Cholesky factorization to accelerate kernel clustering and save memory space. The key idea of the proposed kernel $k$-means clustering using incomplete Cholesky factorization is that we approximate the entire kernel matrix by the product of a low-rank matrix and its transposition. Then linear $k$-means clustering is applied to columns of the transpose of the low-rank matrix. We show both analytically and empirically that the performance of the proposed algorithm is similar to that of the kernel $k$-means clustering algorithm, but our method can deal with large-scale datasets.
\end{abstract}


\end{frontmatter}


\section{Introduction}

Clustering analysis is a classical unsupervised learning method. The aim of clustering is to partition samples into several groups. One advantage of clustering is that it is suitable for processing multi-class datasets. It has been applied in various fields, including image segmentation \cite{liew2003adaptive}, anomaly detection \cite{chandola2009anomaly}, gene sequence analysis \cite{fu2012cd}, market research \cite{kim2008recommender}, etc.

 Clustering has been well studied in recent years, and various clustering algorithms have been proposed one after another \cite{corpet1988multiple, ester1996density, elhamifar2009sparse}. $k$-means clustering \cite{MacQueen1967kmeans} is one of the most popular clustering method since it is simple and efficient in dealing with linear-separable datasets. The target of the $k$-means algorithm is to minimize the Euclidean distance between samples and the clustering centers. The computational complexity of $k$-means is very low and it is suitable to deal with large-scale datasets. However, $k$-means clustering will not achieve satisfactory results when dataset is nonlinear-separable, that is, the dataset cannot be well partitioned into different clusters by hyperplane. To address this deficiency, Sch\"{o}lkopf et al. introduced kernel into $k$-means clustering, and proposed kernel $k$-means clustering \cite{Scholkopf1998kkmeans}. Kernel trick is an effective method to map nonlinear-separable dataset in low-dimensional space to linear-separable dataset in higher-dimensional feature space. By using nonlinear mapping $\phi$, the Euclidean distance between samples in $k$-means is replaced by the kernel distance defined by
 $$d(\phi(\mathrm {x}_i),\phi(\mathrm {x}_j))=k(\mathrm {x}_i,\mathrm {x}_i)-2k(\mathrm{x}_i,\mathrm{x}_j)+k(\mathrm {x}_j,\mathrm {x}_j),$$
where $\mathrm {x}_i\in \Re^d$ and $\mathrm {x}_j\in \Re^d$ are two samples, $k(\cdot,\cdot):\Re^d \times \Re^d\rightarrow \Re$ is the kernel function and $k(\mathrm x_i, \mathrm x_j)=\langle\phi(\mathrm x_i), \phi(\mathrm x_j)\rangle$. However, the full $n \times n$ kernel matrix $K$ is required for kernel $k$-means, because it needs to compute the kernel distance between samples and cluster centers which are a linear combination of all the samples in the feature space. If the number of samples $n$ is very large, computing and storing $K$ consume a lot of memory. Therefore, kernel $k$-means is unsuitable for clustering large-scale datasets.

In this paper, we address this challenge by using low-rank approximation version $\hat{K}$ instead of the full kernel matrix $K$. The low-rank approximation version is generated by incomplete Cholesky factorization (ICF) \cite{sszhou2016}. It iteratively chooses samples one by one into a basic subset $\mathbb{B}$ by minimizing the trace of error, i.e. $\mathrm{tr}(K-\hat{K})$, and finally constructs the rank-$s$ matrix $\hat{K}=PP^\top$, where $s<n$ is the number of elements in the $\mathbb{B}$, $P \in \Re ^{n\times s}$. Then, $k$-means clustering is applied on $P^\top$ to obtain the final cluster results. We show that approximation error of the kernel $k$-means clustering using ICF algorithm reduces exponentially as $s$ increases when the eigenvalues of $K$ decay exponentially.
Moreover, with regard to the Zhou's ICF \cite{sszhou2016}, we show that: (a) for symmetric positive semidefinite (SPSD) matrix $K$, ICF can obtain a rank-$s$ approximation of $K$ after $s$ iterations; (b) let $r$ as the rank of the SPSD matrix $K$, then after $r$ iterations, ICF can successfully sample $r$ linearly independent columns from $K$ without breakdown; (c) the approximation error, $\mathrm{tr}(K-\hat{K})$,  decreases exponentially with the increases of $s$ when the eigenvalues of $K$ decay
exponentially sufficiently fast; (d) the iterations of ICF should be set no more than the rank of $K$, i.e. $s\leq r$; (e) setting the iterations of ICF as $s=O(\log(n/\epsilon))$ can guarantee the error $\mathrm{tr}(K-\hat{K}) <\epsilon$, where $\epsilon$ is a fully small positive number. These claims can ensure the convergence of ICF, which was not discussed in the previous studies. The experimental results illustrate that the accuracy of the proposed algorithm is similar as the kernel $k$-means using entire kernel matrix, but our algorithm can greatly reduce the running time and can process large-scale datasets.

The rest of the paper is organized as follows. Section 2 describes the related work on large-scale kernel $k$-means clustering. Section 3 outlines the kernel $k$-means clustering and the ICF algorithm. Section 4 discusses the convergence of ICF. We present the kernel $k$-means clustering using ICF and its analysis in section 5. Section 6 summarizes the results of the experimental results, and section 7 concludes the study.

\section{Related work}

Kernel $k$-means clustering can achieve higher clustering quality than $k$-means. However, the computational complexity of the kernel $k$-means clustering is high, mainly because the computation and storage of the kernel matrix consume much time and memory. Many algorithms have been proposed in the literature to improve the ability of kernel $k$-means to deal with large-scale datasets (see Table \ref{table_complexity}). \cite{Chitta2011approximate} proposed the approximate kernel $k$-means algorithm. It approximates the cluster centers by the randomly selected samples instead of all the samples to avoid compute the full kernel matrix. The computational complex of this algorithm is $O(s^3+ns^2+Tnsk)$, where $n$ is the number of samples in the dataset, $s$ is the number of selected samples and $T$ is the number of iterations. \cite{Wang2019Nystrom} proposes a scalable kernel $k$-means clustering algorithm which uses Nystr\"{o}m approximation to sample $l$ features, and then the truncated singular values decomposition (SVD) is used to reduce the number of features to $s$. The computational complex of it is $O(ndl+nl^2+Tnsk)$. By employing random Fourier features (RFF) and SVD, \cite{Chitta2012efficient} proposes efficient kernel $k$-means clustering to improve the efficiency of kernel clustering. Its computational complex is $O(nds+ns^2+Tnk^2)$. These algorithms all focus on approximating the kernel matrix and more efficient than the standard kernel $k$-means algorithm whose computational complexity is $O(n^2d+Tn^2k)$.

\begin{table}
\caption{Complexity of kernel $k$-means clustering algorithms. $n$ and $d$ represent the number and dimensionality of data, respectively. $k$ represents the number of clusters. $s<n$ is the number of selected samples and $T$ is the number of iterations in (kernel) $k$-means. For scalable kernel $k$-means clustering, $l>s$ is the number of sampled features by using Nystr\"{o}m method.}
\label{table_complexity}
\centering
\begin{tabular}{|l|l|}
\hline
Clustering algorithms&Complexity\\
\hline
$k$-means \cite{Jain2010kmeans}& $O(Tn^2k)$\\
\hline
Kernel $k$-means \cite{Girolami2002kkmeans}& $O(n^2d+Tn^2k)$\\
\hline
Approximate kernel $k$-means \cite{Chitta2011approximate}& $O(s^3+ns^2+Tnsk)$\\
\hline
Efficient kernel $k$-means \cite{Chitta2012efficient}&$O(nds+ns^2+Tnk^2)$\\
\hline
Scalable kernel $k$-means\cite{Wang2019Nystrom}& $O(ndl+nl^2+Tnsk)$\\
\hline
This paper& $O(ns^2+Tnsk)$\\
\hline
\end{tabular}
\end{table}

Incomplete Cholesky factorization is another effective kernel matrix approximate method. Comparing with data independent method RFF, the data dependent methods ICF has better generalization performance \cite{Yang2012nystromrff}. Comparing with Nystr\"{o}m methods, ICF has three advantages: (a) ICF does not need to compute and store the full kernel matrix, which is not true for some Nystr\"{o}m methods, such as leverage scores sampling \cite{Gittens2016Nystrom}, Farahat schemes \cite{Farahat2011Nystrom} and landmark point sampling with $k$-means \cite{Zhang2008Nystrom}; (b) the result of ICF is deterministic, in contrast, some random-based Nystr\"{o}m method randomly selects a subset of training samples as the basis functions and RFF randomly samples vectors from a distribution to form the basis functions; (c) ICF does not sample the same column of kernel matrix twice, whereas most non-deterministic Nystr\"{o}m methods require sampling with replacement \cite{Patel2016ICD}. ICF has been successfully applied in several kernel-based algorithms to improve computational efficiency or enhance sparsity, for example, Zhou \cite{sszhou2016}obtains the sparse solution of least square support vector machine (LSSVM) by using ICF, Chen et al. \cite{chen2018rlssvm} introduce ICF into robust LSSVM and enables it to classify and regress large-scale datasets with noise, Frederix et al. \cite{Frederix2013icf} propose a sparse spectral clustering method with ICF, etc. ICF with different pivot selection rules are presented, see \cite{Bach2002icf, Harbrecht2012icf, Patel2016ICD}. Recently, Zhou proposed an improved ICF with a new pivot selection rule \cite{sszhou2016}. One advantage of this method is that in each iteration, only the diagonal elements of error matrix are required, not the whole matrix. In the sequel, when we refer to ICF, we mean Zhou's method.

\section{Background}

In this section, we first describe the kernel $k$-means clustering, and then describe incomplete Cholesky factorization method.


\subsection{Kernel $k$-means clustering}

Let $X = \{\mathrm x_1, \mathrm x_2, ..., \mathrm x_n\}$ be the input dataset consisting of $n$ samples, where $\mathrm x_i \in \Re^d$. The dataset can be participated into $k$ clusters: $\{C_1,\ldots, C_k\}$. $|C_i|$ be the number of samples in cluster $C_i$. The objective of kernel $k$-means clustering is to minimize the sum of kernel distances between each sample and the center of the cluster to which the sample belongs, that is, to minimize the following optimization problem:
\begin{equation}\label{eq:kkmeans}
 \mathop{\arg\min}\limits_{C_1,\ldots, C_k} \frac{1}{n}\sum\limits_{i=1}^k\sum\limits_{j\in C_i}\|\phi(\mathrm x_j)-\frac{1}{|C_i|}\sum\limits_{l\in C_i}\phi(\mathrm x_l)\|_2^2
\end{equation}
where $\phi(\cdot)$ is the mapping to project the samples to high-dimensional feature space. In fact, we usually do not need to know what the mapping $\phi(\cdot)$ is, because they often appear in the form of inner product $\langle \phi(\mathrm x_i), \phi(\mathrm x_j)\rangle$. We denote it as $k(\mathrm x_i, \mathrm x_j)$, i.e. $k(\mathrm x_i, \mathrm x_j)=\langle \phi(\mathrm x_i), \phi(\mathrm x_j)\rangle$, where $k(\cdot, \cdot)$ is called the kernel function. The factor $1/n$ is introduced only for the purpose of normalization.

Let $K \in \Re^{n\times n}$ be the kernel matrix with $K_{ij} = k(\mathrm x_i, \mathrm x_j)$, and $K=UDU^\top$ be the full eigenvalue decomposition of $K$. Denote $\mathrm k_1, \mathrm k_2,\ldots,\mathrm k_n \in \Re^n$ be the columns of the $D^{1/2}U^\top\in \Re^{n \times n}$, then the problem \eqref{eq:kkmeans} is equivalent to the following optimization problem \cite{Wang2019Nystrom}:
\begin{equation}\label{eq:kkmeans2}
 \mathop{\arg\min}\limits_{C_1,\ldots, C_k} \frac{1}{n}\sum\limits_{i=1}^k\sum\limits_{j\in C_i}\|\mathrm k_j-\frac{1}{|C_i|}\sum\limits_{l\in C_i}\mathrm k_l\|_2^2
\end{equation}

Problem \eqref{eq:kkmeans2} can be solved by $k$-mean clustering algorithm. However, the entire kernel matrix $K$ and its eigenvalue decomposition must be computed and stored in advance, which cost $O(n^2d)$ and $O(n^3)$ running time, respectively. Therefore, the total computational complexity of solving optimization problem \eqref{eq:kkmeans2} is $O(n^3 + n^2d + Tn^2k)$, where $T$ is the number of iterations of $k$-means algorithm. When dataset contains points greater than a few ten thousands, the computational cost of this method is very high. The goal of this paper is to reduce both the computational complexity and the memory requirements of kernel $k$-means clustering.
\subsection{Incomplete Cholesky factorization}

For a positive semi-definite matrix $K\in\Re^{n\times n}$, the incomplete Cholesky factorization constructs a low-rank approximate matrix $\hat{K}=K_{\mathbb{MB}}K_{\mathbb{B}\mathbb{B}}^{-1}K_{\mathbb{MB}}^\top=PP^\top$ for $K$, where $\mathbb{M}=\{1,2,\ldots,n\}$ is the column/row indices of $K$, $\mathbb{B}$ is a subset of $\mathbb{M}$, which contains $s$ indices of the selected columns, $K_{\mathbb{MB}}\in \Re^{n\times s}$ denotes a sub-matrix of $K$ composed by the selected columns whose indices are in $\mathbb{B}$, and $K_{\mathbb{B}\mathbb{B}}\in \Re^{s\times s}$ is a square sub-matrix of $K_{\mathbb{MB}}$ composed by the selected rows whose indices are in $\mathbb{B}$. The subset $\mathbb{B}$ is generated as follows. ICF first sets $\mathbb{B}$ as an empty set, and then iteratively joins the index corresponding to the largest diagonal entry of error matrix $E^i=K-\hat{K}^i=K-K_{\mathbb{MB}_i}K_{\mathbb{B}_i\mathbb{B}_i}^{-1}K_{\mathbb{MB}_i}^\top$ into $\mathbb{B}$ to minimize $\mathrm {tr}(E^i)$ until the termination condition is satisfied,
where $\mathbb{B}_i$ and $E^i$ are the $\mathbb{B}$ and the error matrix $E$ in the $i$-th iteration, respectively. Minimizing $\mathrm {tr}(E)$ also implies minimizing trace norm, and the upper bounds of $\|E\|_2$ and $\|E\|_F$, because $\|E\|_2\leq \|E\|_F \leq \|E\|_\ast =\mathrm {tr}(E)$ for positive semi-definite matrix $E$.
$P\in \Re^{n\times s}$ is updated by the following theorem \cite{sszhou2016}.

\begin{thm}\label{th:icf}
 Denote $\mathbb{N}_i=\mathbb{M}\setminus \mathbb{B}_i$. Let $t=\arg\max_{j\in \mathbb{N}_i}E^i_{jj}$ be the selected index into $\mathbb{B}_i$ and $K_{\mathbb{M}t}:=[k(\mathrm x_1,\mathrm x_t),k(\mathrm x_2,\mathrm x_t),\ldots,k(\mathrm x_n,\mathrm x_t)]^\top$. Set $\mathbb{B}_{i+1}=\mathbb{B}_i \cup \{t\}$ and $\mathbb{N}_{i+1}=\mathbb{M}\setminus \mathbb{B}_{i+1}$. If $\hat{K}^i=P^i{P^i}^\top$, then $\hat{K}^{i+1}=P^{i+1}{P^{i+1}}^\top$ with $P^{i+1}=[P^i,~\mathrm p]$, where $\mathrm p=\nu^\top(K_{\mathbb{M}t}-P^i \mathrm u)$, $\mathrm u^\top=P_{t,\cdot}^i$ being the $t$-th row of $P^i$ and $\nu=(K_{tt}-\mathrm u^\top \mathrm u)^{1/2}$. Furthermore, $E_{jj}^{i+1}=(K-\hat{K}^{i+1})_{jj}=E_{jj}^i-\mathrm p_j^2$, $j\in \mathbb{N}_{i+1}$.
\end{thm}

The ICF algorithm \cite{sszhou2016} is listed in Algorithm \ref{alg:icf}.

\begin{algorithm}[ht]
\renewcommand{\algorithmicrequire}{\textbf{Input:}}
\renewcommand\algorithmicensure {\textbf{Output:} }
\caption{\textbf{Incomplete Cholesky Factorization}}
\begin{algorithmic}[1]\label{alg:icf}
\REQUIRE Dataset $X\in\Re^{n\times d}$, kernel function $k(\cdot, \cdot)$, the iteration numbers $maxiter$
 and a sufficiently small positive number $\epsilon$.
\ENSURE $\mathbb{B}$ and matrix $P$.
\STATE $\mathbb{B}_0=\emptyset$, $\mathbb{N}_0=\mathbb{M}$, $\mathrm e^0=[k(\mathrm x_1,\mathrm x_1),\ldots,k(\mathrm x_n,\mathrm x_n)]^\top$, $\varepsilon_0=\|\mathrm e^0\|_1$. Set $i=0$;
\WHILE{$\varepsilon_i>\epsilon$ and $i\le maxiter$}
\STATE $t=\arg\max_{j\in\mathbb{N}_i}{e^i_j}$, $\mathbb{B}_{i+1}:=\mathbb{B}_{i}\cup\{t\}$, $\mathbb{N}_{i+1}:=\mathbb{N}_i\setminus\{t\}$; Calculate $K_{\mathbb{M} t}$;
\IF{$i=0$}
\STATE $P^{i+1}:= K_{\mathbb{M}t}/\sqrt{K_{tt}}$;
\ELSE
\STATE Calculate $P^{i+1}$ by Theorem \ref{th:icf};
\ENDIF
\STATE $e^{i+1}_j:=e^i_j-p_j^2$, $j\in \mathbb{N}_{i+1}$;\\ $\varepsilon_{i+1}:=\sum_{j\in \mathbb{N}_{i+1}}e^{i+1}_j$;\\$i:=i+1$;
\ENDWHILE
\RETURN $\mathbb{B}\leftarrow \mathbb{B}_i$ and $P\leftarrow P^i$.
\end{algorithmic}
\end{algorithm}

Set $P_{\mathbb{B}}$ as a sub-matrix of $P$ composed by rows corresponding to $\mathbb{B}$. After obtaining $P$ by Algorithm \ref{alg:icf}, $K_{\mathbb{MB}}=PP_{\mathbb{B}}^\top$ and $K_{\mathbb{BB}}=P_{\mathbb{B}}P_{\mathbb{B}}^\top$. In the Algorithm \ref{alg:icf}, only the diagonal entries of the error matrix $E^i$ are required, and the total cost is just $O(ns^2)$.

\section{The convergence of ICF}

In this section, we discuss the convergence of ICF, which is not discussed in previous literature. Firstly, we will prove that ICF produces a rank $s$ approximate matrix $\hat{K}$ after $s$ iterations. Secondly, we show that ICF can generate a rank $r$ matrix in $r$ iterations without breakdown. At last, the convergence of ICF will be verified.

\begin{thm}\label{th:line_independent}
 For symmetric positive semidefinite (SPSD) matrix $K$, ICF generates a rank $s$ approximate matrix for $K$ after $s$ iterations.
\end{thm}
\begin{proof}
We use inductive method to prove this theorem. First, we set the number of iteration $i=1$, then $\hat{K}^1=\max\limits_{j\in \{1,\dots,n\}}\{K_{jj}\}\neq 0$ as $K$ is a SPSD matrix, therefore the rank of $\hat{K}^1$ is 1.

Next, we assume that the rank of $\hat{K}^i$ is $i$, and then prove that the rank of $\hat{K}^{i+1}$ is $i+1$.

Suppose that ICF can select the $(i+1)$-th element into $\mathbb{B}_i$, that is, there exists $t={\arg\max}_{j\in \mathbb{N}_i}E_{jj}^i$, then $\mathbb{B}_{i+1}=\mathbb{B}_i \cup \{t\}$, $\mathbb{N}_{i+1}=\mathbb{N}_i \setminus \{t\}$, and $E_{tt}^i=(K-\hat{K}^i)_{tt}=K_{tt}-K_{t\mathbb{B}_i}K_{\mathbb{B}_i\mathbb{B}_i}^{-1} K_{t\mathbb{B}_i}^\top \neq 0$.

Set $L^i{L^i}^\top$ be the Cholesky factorization of $K_{\mathbb{B}_i\mathbb{B}_i}$ and $L^{i+1}{L^{i+1}}^\top$ be the Cholesky factorization of $K_{\mathbb{B}_{i+1}\mathbb{B}_{i+1}}$, where
\begin{equation}\label{eq:L}
L^{i+1}=\left[\begin{array}{ccc}
L^i&0\\
\mathrm u^\top& \nu\end{array}\right].
\end{equation}
Because
\begin{equation*}
\begin{array}{lll}
K_{\mathbb{B}_{i+1}\mathbb{B}_{i+1}}&={L^{i+1}}{L^{i+1}}^\top\\
&=\left[\begin{array}{ccc}
L^i&0\\
\mathrm u^\top& \nu\end{array}\right]\left[\begin{array}{ccc}
{L^i}^\top&\mathrm u\\
0& \nu\end{array}\right]\\
&=\left[\begin{array}{ccc}L^i{L^i}^\top&L^i \mathrm u\\ \mathrm u^\top {L^i}^\top&\mathrm u^\top \mathrm u+\nu^2\end{array}\right]\\
&=\left[\begin{array}{ccc} K_{\mathbb{B}_{i}\mathbb{B}_{i}}&K_{t\mathbb{B}_{i}}^\top\\ K_{t\mathbb{B}_{i}}&K_{tt}\end{array}\right],
\end{array}
\end{equation*}
we obtain
$$\nu^2=K_{tt}-\mathrm u^\top \mathrm u=K_{tt}-K_{t\mathbb{B}_{i}} K_{\mathbb{B}_{i}\mathbb{B}_{i}}^{-1} K_{t\mathbb{B}_{i}}^\top=E_{tt}^i\neq 0.$$
Therefore, $L^{i+1}$ is a full rank matrix. The ranks of $K_{\mathbb{B}_{i+1}\mathbb{B}_{i+1}}$ and $K_{\mathbb{M}\mathbb{B}_{i+1}}$ are both $i+1$, hence the rank of $\hat{K}^{i+1}=K_{\mathbb{MB}_{i+1}}K_{\mathbb{B}_{i+1}\mathbb{B}_{i+1}}^{-1}K_{\mathbb{MB}_{i+1}}^\top$ is $i+1$. The theorem is proven.
\end{proof}

Theorem \ref{th:line_independent} indicates that the columns of $K_{\mathbb{MB}}$ are linear independent, that is, the selected $s$ columns from $K$ corresponding to the $\mathbb{B}$ are linear independent. Next, we will show that ICF cannot breakdown before selecting $s$ linear independent columns from $K$, where $s\leq r$.

\begin{thm}
Set the rank of SPSD matrix $K$ as $r$. ICF can sample $r$ linear independent columns from $K$ after $r$ iterations.
\end{thm}
\begin{proof}
Theorem \ref{th:line_independent} has proven that ICF can generate a rank $r$ approximate matrix for $K$ after $r$ iterations. Next, we use reduction to absurdity to prove that ICF does not breakdown before $r$ iterations.

Assume for any $j\in \mathbb{N}_i$, $E_{jj}^i=0$ after $i$ iterations, $i<r$, then for $t\in \mathbb{N}_i$, $E_{tt}^i=0$. Set $\mathbb{B}_{i+1}=\mathbb{B}_i\cup \{t\}$, then $K_{\mathbb{B}_{i+1}\mathbb{B}_{i+1}}$ is not a full rank matrix as $\nu=0$ in \eqref{eq:L}.

Because $K$ is a rank $r$ SPSD matrix, $K$ has eigenvalue decomposition $K=U D U^\top$, where $D$ is a $r\times r$ diagonal matrix, $U$ is a $n\times r$ column orthogonal matrix. Denote $U_{\mathbb{B}_{i+1}}$ as a sub-matrix of $U$. It is comprised by the rows of $U$ with row indices corresponding to $\mathbb{B}_{i+1}$. 
Then $K_{\mathbb{B}_{i+1}\mathbb{B}_{i+1}}$ has the decomposition $K_{\mathbb{B}_{i+1}\mathbb{B}_{i+1}}=U_{\mathbb{B}_{i+1}}D U_{\mathbb{B}_{i+1}}^\top$. $U_{\mathbb{B}_{i+1}}$ is not a full rank matrix, because $K_{\mathbb{B}_{i+1}\mathbb{B}_{i+1}}$ is not a full rank matrix. Therefore, the $i+1$ row of $U_{\mathbb{B}_{i+1}}$ can be represented linearly by the first $i$ rows. This conclusion is held for any $t$ in $\mathbb{N}_i$, hence every row in $U$ can be linearly represented by $U_{\mathbb{B}_i}$, and the rank of $U$ is at most $i<r$. This contradicts that $U$ is a rank $r$ matrix. Therefore, the assume is not invalid, and there exists at least one $j$ satisfying $E_{jj}^i\neq 0$ after $i$ iterations, where $i<r$ and $j\in\mathbb{N}_i$. In other words, ICF cannot breakdown before $r$ iterations. The theorem is proven.
\end{proof}

\begin{cor}
The number of iterations in ICF should be set no more than the rank of $K$.
\end{cor}
\begin{proof}
Assume there exists $t$ satisfying $E_{tt}^r=\max_{j\in \mathbb{N}_r}E^r_{jj}\neq 0$, then $\mathbb{B}_{r+1}=\mathbb{B}_{r}\cup \{t\}$, and the ranks of $K_{\mathbb{MB}_{r+1}}$ and $K_{\mathbb{B}_{r+1}\mathbb{B}_{r+1}}$ are both $r+1$. Therefore, the rank of $\hat{K}=K_{\mathbb{MB}_{r+1}}K_{\mathbb{B}_{r+1}\mathbb{B}_{r+1}}^{-1}K_{\mathbb{MB}_{r+1}}^\top$ is $r+1$, which contradicts that $K$ is a rank $r$ matrix. So the assumption is invalid. The number of iterations should be set no more than $r$.
\end{proof}

\begin{thm}\label{tm:convergence}
The approximation error $\mathrm {tr}(K-\hat{K})=\mathrm {tr}(K-PP^\top)$ decreases monotonously as the number of iterations increases.
\end{thm}
\begin{proof}
From Theorem \ref{th:icf}, we obtain $E_{jj}^{i+1}=E_{jj}^i- p_j^2$, $j\in \mathbb{N}_{i+1}$. Therefore,
$$E_{jj}^{s}=E_{jj}^0-  p_{j_1}^2-  p_{j_2}^2-\ldots- p_{j_s}^2,$$
where $j\in \mathbb{N}_s$ and $j_i\in \mathbb{N}_i$ for $i=1,\ldots,s$. The total error $\mathrm{tr} (K-PP^\top)=\sum\limits_{j\in \mathbb{N}_s}E_{jj}^s=\sum\limits_{j\in \mathbb{N}_s}\left[K_{jj}-\sum\limits_{i=1}^s p_{j_i}^2\right]$.

When $s$ increases, $\sum\limits_{i=1}^s p_{j_i}^2$ increases. Hence $K_{jj}-\sum\limits_{i=1}^s p_{j_i}^2$ decreases. Moreover, the number of elements in $\mathbb{N}$ declines as $s$ increases. Therefore, the total approximation error decreases monotonously with the increasing of $s$.
\end{proof}

Theorem \ref{tm:convergence} only shows that the approximation error $\mathrm {tr}(K-\hat{K})$ decreases as $s$ increases, and when $s=r$, then $\mathrm {tr}(K-\hat{K})=0$. However, it does not give the decline rate. The following theorem indicates that the error decreases exponentially when the eigenvalues of $K$ decay exponentially sufficiently fast.

\begin{thm}\label{tm:convergence_2}
Denote $\hat{K}$ be the rank-$s$ approximation of $K$, $\lambda_s (K)$ be the $s$-th largest eigenvalue of matrix $K\in\Re^{n\times n}$. Assume
$$\lambda_s (K) \leq C4^{-s}\exp(-bs),$$
for some $C,~b > 0$ uniformly in $n$. Then, $s=O(\log(n/\epsilon))$ satisfies $\mathrm{tr}(K-\hat{K})< \epsilon$, where $\epsilon>0$ is a constant.
\end{thm}
\begin{proof}

Because $P_\mathbb{B}P_\mathbb{B}^\top$ is the (incomplete) Cholesky factorization of $K_\mathbb{BB}$, we have
$$\frac{1}{\lambda_s(K_\mathbb{BB})}=\|K_{\mathbb{BB}}^{-1}\|_2=\|P_\mathbb{B}^{-1}\|_2^2\leq \frac{4^s+6s-1}{9(P_\mathbb{B})_{ss}^2}\leq \frac{4^s}{(P_\mathbb{B})_{ss}^2}.$$
In the above formula, the first inequality holds according to \cite{Harbrecht2012icf}. Denote $t= \arg\max _{j\in\mathbb{N}_s}e_j^s$, then $$e_t^s=K_{tt}-P_{t1}^2-P_{t2}^2-\ldots-P_{ts}^2= p_t^2=(P_\mathbb{B})_{ss}^2. $$
The approximation error $\mathrm{tr}(K-\hat{K})$ is bounded by
\begin{equation}\label{eq:tr}
\begin{split}
\mathrm{tr}(K-\hat{K})&=\sum_{j\in \mathbb{N}_s}e_j^s\\
&\leq(n-s)(P_\mathbb{B})_{ss}^2\\
&\leq(n-s) 4^s\lambda_s(K_{\mathbb{BB}})\\
&\leq n 4^s\lambda_s(K)\\
&\leq nC\exp(-bs).
\end{split}
\end{equation}
This implies $s=O(\log({n/\epsilon}))$. The theorem is proven.
\end{proof}

In order to further verify the Theorem \ref{tm:convergence} and \ref{tm:convergence_2}, we conducted experiments on datasets USPS and MNIST by using ICF. The eigenvalues of $K$ are decay exponentially for these two datasets. We applied Gaussian kernel function in experiments, $k(\mathrm x_i, \mathrm x_j)=\exp(-\sigma\|\mathrm x_i-\mathrm x_j\|^2)$, where $\sigma$ is the parameter of the kernel function, which were set as $2^{-10}$ and $2^{-6}$ for these two datasets, respectively. Fig. \ref{fig:icf_error} gives the experimental results. It shows that the approximation error decreases exponentially as the increasing of iterations $s$.
Theorem \ref{tm:convergence_2} and the experimental results in Fig. \ref{fig:icf_error} show that ICF is a exponential convergence algorithm. As the approximation error drops rapidly at first, in fact, it is sufficient to set $s$ as a few hundred number to get satisfy results, where $s\leq r$.
\begin{figure}[h]
\centering
 \subfigure[USPS]{
    \includegraphics[width=0.48\textwidth]{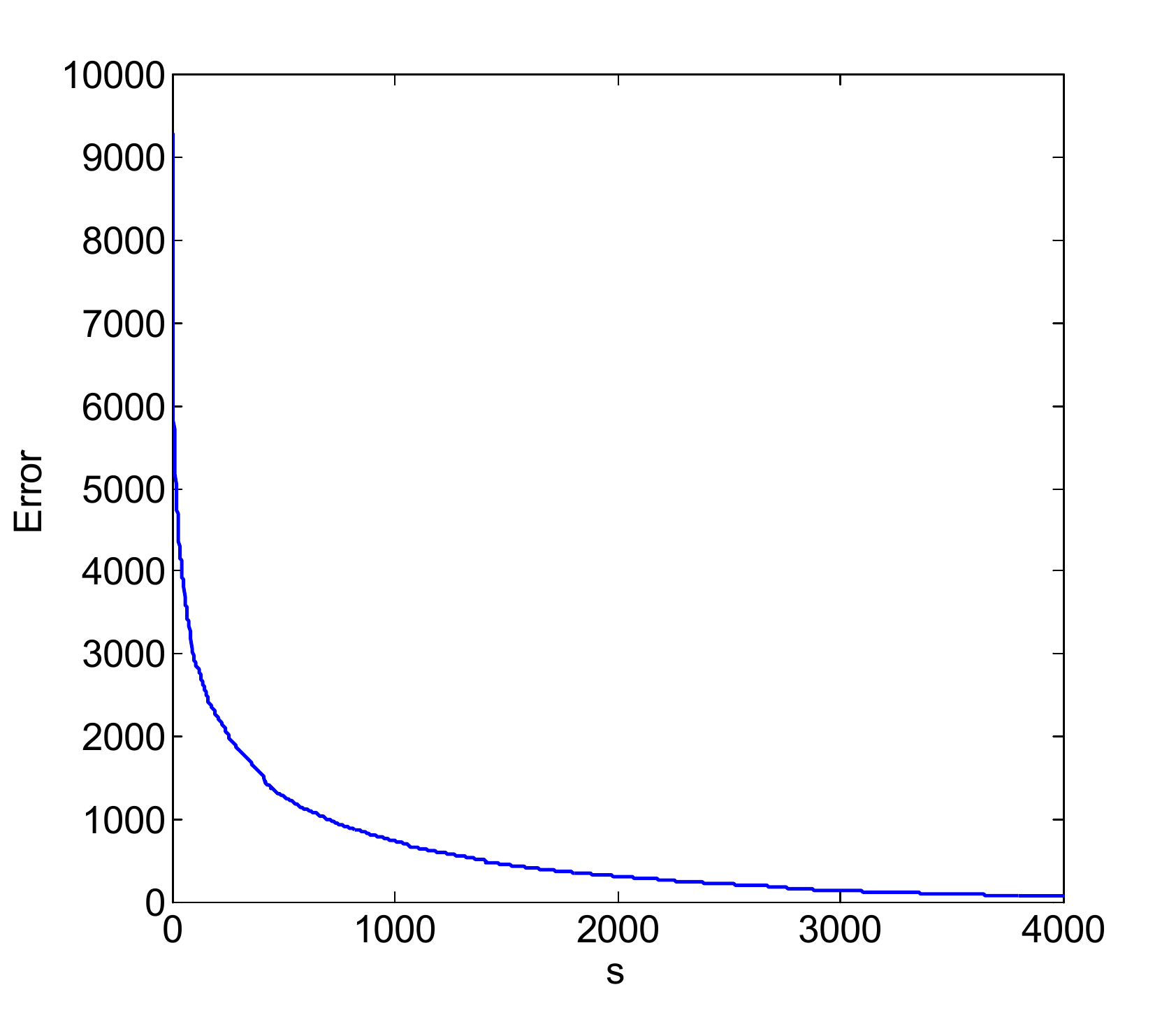}} 
 \subfigure[MNIST]{
    \includegraphics[width=0.48\textwidth]{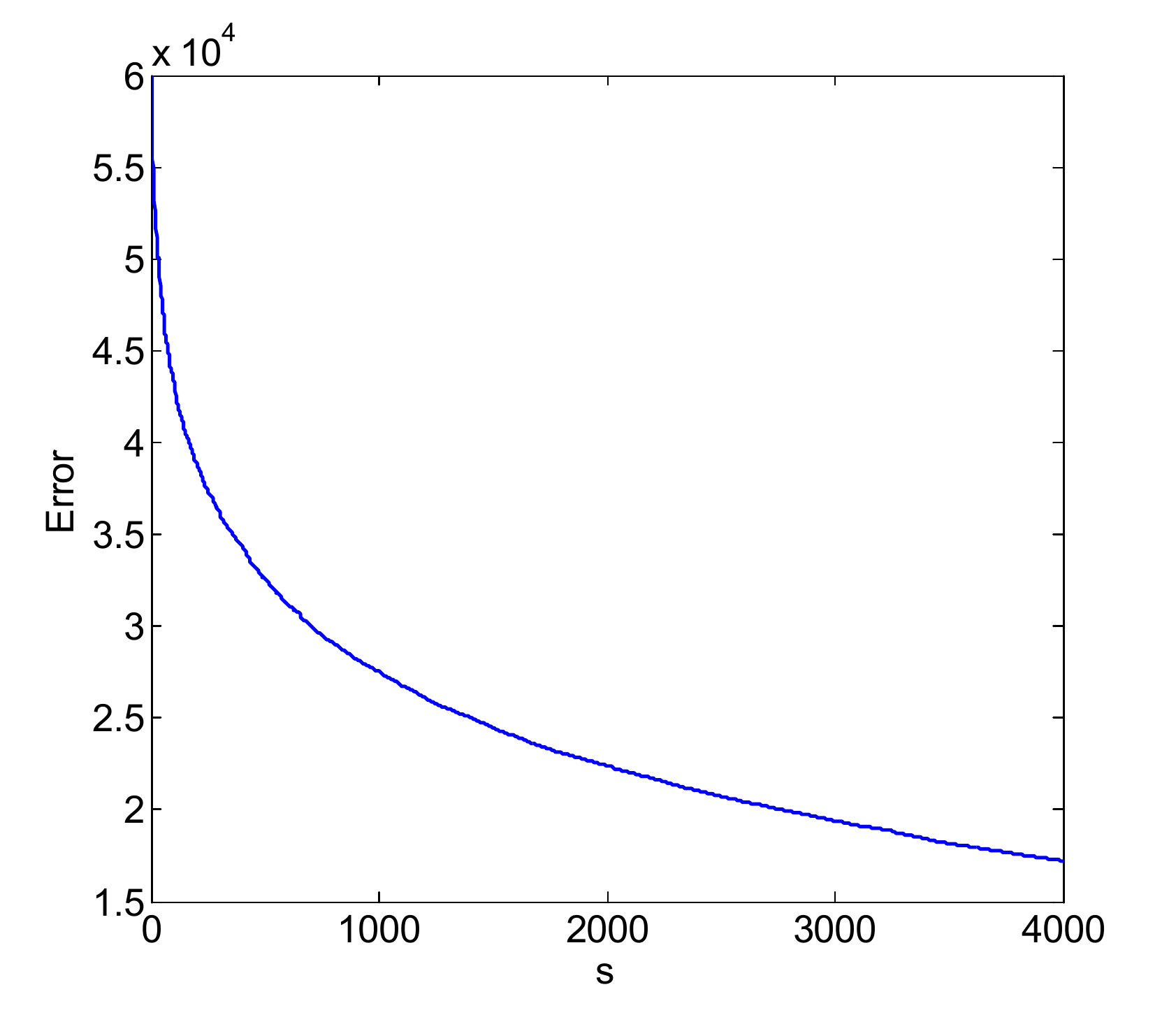}}
 \caption{The approximation error between $K$ and $\hat{K}$. The $\hat{K}$ is generated by ICF. The $x$-axis is the rank of $\hat{K}$, and the $y$-axis is the error $\mathrm {tr}(K-\hat{K})$. The error decreases exponentially as $s$ increases.}
 \label{fig:icf_error}
\end{figure}
\section{Kernel $k$-means clustering using incomplete Cholesky factorization}

The runtime complexity of kernel $k$-means clustering is very high, which leads to kernel $k$-means algorithms running slowly and can not deal with large-scale datasets. The main reason is that standard kernel $k$-means algorithm needs to compute entire kernel matrix. In this section, we apply incomplete Cholesky factorization method to obtain the low rank approximation of the kernel matrix. The new algorithm avoids computing the entire kernel matrix, reducing the computational complexity and the storage space of kernel $k$-means clustering, but it can achieve comparable clustering performance as standard kernel $k$-means using the entire kernel matrix.

ICF algorithm outputs a matrix $P$ so that kernel matrix $K\approx PP^\top$. When $K$ has the eigenvalue decomposition $K=UDU^\top=(UD^{\frac{1}{2}})(UD^{\frac{1}{2}})^\top$, $P^\top\approx D^{\frac{1}{2}}U^\top$. Therefore, the problem \eqref{eq:kkmeans2} is equivalent to the following optimization problem:
\begin{equation}\label{eq:kkmeans3}
 \mathop{\arg\min}\limits_{C_1,\ldots, C_k} \frac{1}{n}\sum\limits_{i=1}^c\sum\limits_{j\in C_i}\| P^\top_{\cdot, j}-\frac{1}{|C_i|}\sum\limits_{l\in C_i}P^\top_{\cdot, l}\|_2^2
\end{equation}
where $P^\top_{\cdot, j}$ denotes the columns of $P^\top$.
The idea of the kernel $k$-means clustering using incomplete Cholesky factorization is that run Algorithm \ref{alg:icf} first to obtain matrix $P$, and then run $k$-means clustering algorithm using the columns of $P^\top$ as the input data to get the clustering results.  Algorithm \ref{alg:kkmeans-icf} lists the new algorithm in detail.

\begin{algorithm}[ht]
\renewcommand{\algorithmicrequire}{\textbf{Input:}}
\renewcommand\algorithmicensure {\textbf{Output:} }
\caption{\textbf{Kernel $k$-means Clustering Using Incomplete Cholesky Factorization}}
\begin{algorithmic}[1]\label{alg:kkmeans-icf}
\REQUIRE Dataset $X\in\Re^{n\times d}$, kernel function $k(\cdot, \cdot)$, target dimensions $s$, number of clusters $k$.
\ENSURE The clustering results.
\STATE Run ICF algorithm to get the matrix $P\in \Re^{n\times s}$ such that $K\approx PP^\top$;
\STATE Perform $k$-means clustering over the columns of $P^\top$ to obtain the clustering results.
\end{algorithmic}
\end{algorithm}

In the following, we bound the difference between the solutions of \eqref{eq:kkmeans2} and \eqref{eq:kkmeans3}. Firstly, we give the equivalent form of \eqref{eq:kkmeans2}. Let $V\in\Re^{n\times k}$ is a indicator matrix which has one non-zero element per row. When the $i$-th sample belongs to the $j$-th cluster, then $V_{ij}=1/\sqrt{|C_j|}$, where $|C_j|$ denotes the number of samples in cluster $j$. Note that $V^\top V \in\Re^{k\times k}$ is an identity matrix and $VV^\top\in\Re^{n\times n}$ is a symmetric matrix. Then,
 \begin{equation*}
 \begin{split}
 &\sum\limits_{i=1}^k\sum\limits_{j\in C_i}\|\mathrm k_j-\frac{1}{|C_i|}\sum\limits_{l\in C_i}\mathrm k_l\|_2^2=\|(D^{\frac{1}{2}}U^\top)- (D^{\frac{1}{2}}U^\top) VV^\top\|_F^2\\
 &=\mathrm {tr}((I-VV^\top)K(I-VV^\top))=\mathrm {tr}(K)-\mathrm {tr}(V^\top KV),
 \end{split}
 \end{equation*}
 where $I$ is a $n$ identity matrix. Therefore, \eqref{eq:kkmeans2} is equivalent to
\begin{equation}\label{eq:kkmeans_tr}
\mathop{\arg\max}\limits_V ~\frac{1}{n}\mathrm {tr}(V^\top K V).
\end{equation}
According to \eqref{eq:kkmeans_tr}, the equivalent form of \eqref{eq:kkmeans3} is
\begin{equation}\label{eq:kkmeans_tr2}
\mathop{\arg\max}\limits_V ~\frac{1}{n}\mathrm {tr}(V^\top \hat{K} V),
\end{equation}
where $\hat{K}=PP^\top$ is the approximation matrix of $K$. The following theorem bounds the difference between the solutions of \eqref{eq:kkmeans_tr} and \eqref{eq:kkmeans_tr2}.

\begin{thm}\label{th:convergence_kkmeans}
Let $V^\ast$ and $\hat{V}^\ast$ be the optimal solutions of \eqref{eq:kkmeans_tr} and \eqref{eq:kkmeans_tr2}, respectively. Assume the eigenvalues of $K$ decay exponentially: $\lambda_s (K) \leq C4^{-s}\exp(-bs)$
for some $C,~b > 0$ uniformly in $n$. We have
\begin{equation}
\frac{1}{n}\mathrm {tr}\left([V^\ast-\hat{V}^\ast]^\top K [V^\ast-\hat{V}^\ast]\right)\leq 2\sqrt{k}C \exp(-bs).
\end{equation}
\end{thm}
\begin{proof}
Because
\begin{equation*}
\begin{split}
\mathrm {tr}[(V^\ast)^\top(K-\hat{K})V^\ast]&=\mathrm {tr}[V^\ast (V^\ast)^\top(K-\hat{K})]\\
&\leq \sqrt{\mathrm{tr}[V^\ast (V^\ast)^\top]^2}\sqrt{\mathrm{tr}(K-\hat{K})^2}\\
&\leq \sqrt{k}~\mathrm {tr}(K-\hat{K}),
\end{split}
\end{equation*}
we have
\begin{equation*}
\begin{split}
\mathrm {tr}[(V^\ast)^\top KV^\ast]&\leq \mathrm {tr}[(V^\ast)^\top \hat{K}V^\ast]+\sqrt{k}~\mathrm {tr}(K-\hat{K})\\
&\leq \mathrm {tr}[(\hat{V}^\ast)^\top \hat{K}\hat{V}^\ast]+\sqrt{k}~\mathrm {tr}(K-\hat{K})\\
&\leq \mathrm {tr}[(\hat{V}^\ast)^\top {K}\hat{V}^\ast]+2\sqrt{k}~\mathrm {tr}(K-\hat{K}).
\end{split}
\end{equation*}
Therefore,
$$\mathrm {tr}\left([V^\ast-\hat{V}^\ast]^\top K [V^\ast-\hat{V}^\ast]\right)\leq 2\sqrt{k}~\mathrm {tr}(K-\hat{K}).$$
The proof is completed by using \eqref{eq:tr}.
\end{proof}

Theorem \ref{th:convergence_kkmeans} indicates that the approximation error of the kernel $k$-means clustering using ICF reduces as $s$ increases. The rate of decline is exponential.

\textbf{Computational Complexity.} The Algorithm \ref{alg:kkmeans-icf} only consists of two steps. The first step is to perform ICF algorithm, the complexity of which is $O(ns^2)$ \cite{sszhou2016}. The second step is to run $k$-means clustering on $n\times s$ matrix $P$, which takes $O(Tnsk)$ time, where $T$ is the number of iterations required for convergence. Hence, the total computational complexity of Algorithm \ref{alg:kkmeans-icf} is $O(ns^2+Tnsk)$. By comparison, directly solving \eqref{eq:kkmeans2} by using entire kernel matrix takes $O(n^3+ n^2d + Tn^2k)$ time. Therefore, our ICF-based method greatly reduces the computational complexity.

\section{Experiments}

In order to measure the performance of the new algorithm, we compare our proposed algorithm with kernel $k$-means clustering and some of its improved algorithms in terms of clustering accuracy and time consumption. The first set of experiments was carried on three $2$-dimensional synthetic datasets to show that the new algorithm can cluster non-linear data points well. The second set of experiments performs on several real-world datasets. The experimental results on medium-sized datasets demonstrate that the proposed algorithm is not only faster than the kernel $k$-means clustering, but also can obtain as good performance as the kernel $k$-means algorithm in terms of clustering accuracy. For large-sized real-world datasets, the full kernel matrix is infeasible, we only compare the performance of the proposed algorithm with improved $k$-means algorithms. Gaussian kernel function $k(\mathrm x_i, \mathrm x_j)=\exp(-\sigma \|\mathrm x_i-\mathrm x_j\|^2)$ was used for all the datasets. All algorithms were implemented in MATLAB and run on a 2.40 GHz processor with 8 GB RAM.

\subsection{Synthetic datasets experiments}

In order to show the clustering effect of the proposed algorithm, we generate three datasets named Ring, Parabolic and Zigzag, which cannot be clustered well by $k$-means algorithm. Each dataset contains two clusters, and each cluster contains $500$ data points. The number of sampled data points is set as $50$ in ICF for all the datasets. The parameter $\sigma$ in the Gaussian kernel function is set as $2^4$, $2^{1}$ and $2^{3}$ for Ring, Parabolic and Zigzag datasets, respectively. Fig. \ref{fig:Synthetic} gives the experimental results. Fig. \ref{fig:Synthetic} illustrates that the new algorithm can cluster data points well even only using $5\%$ points. The clustering accuracy for each dataset is $100\%$. This validates the performance of the new algorithm very well.

\begin{figure}
\centering
 \subfigure[Ring]{
    \includegraphics[width=0.31\textwidth]{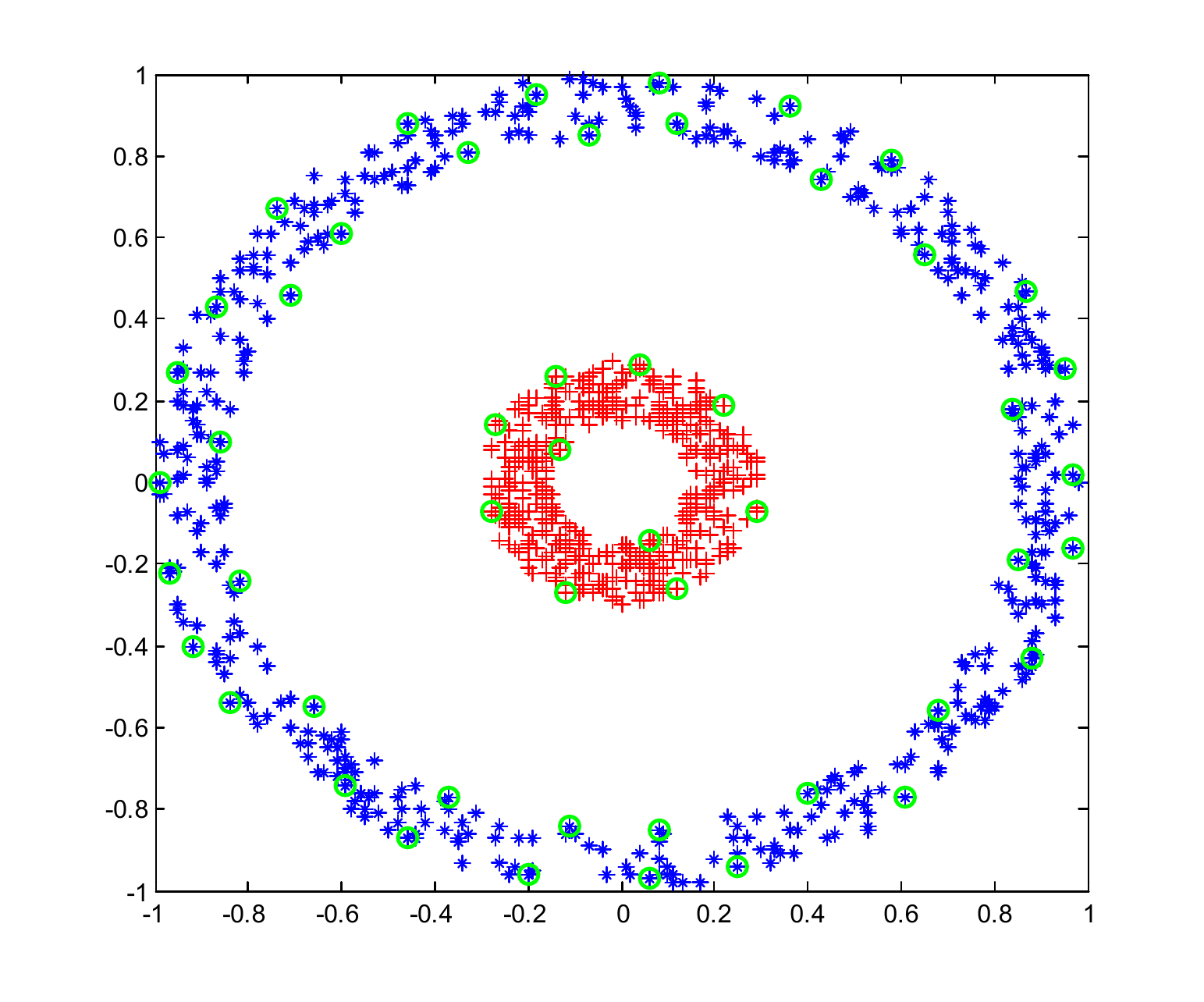}} 
 \subfigure[Parabolic]{
    \includegraphics[width=0.31\textwidth]{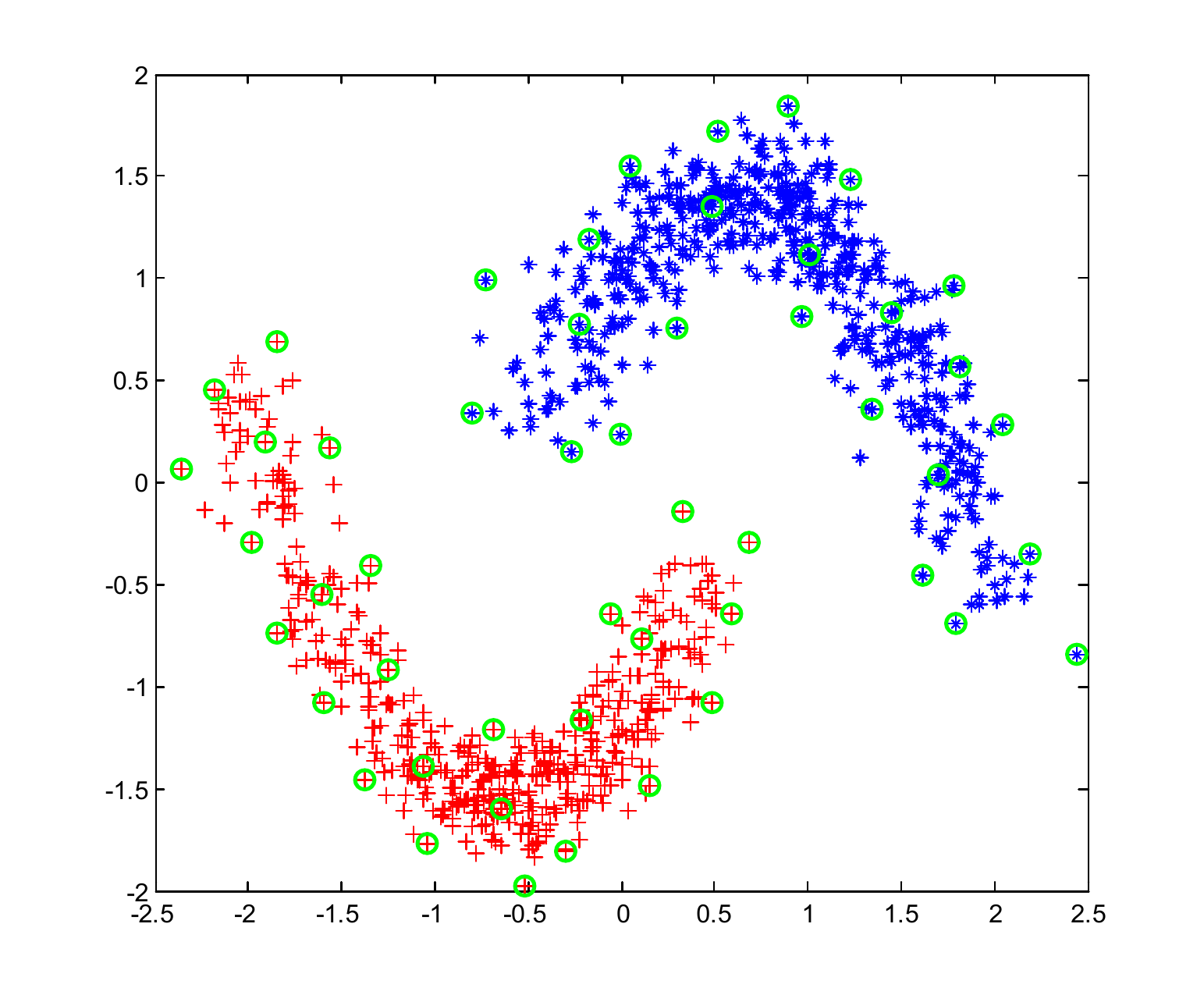}}
 \subfigure[Zigzag]{
    \includegraphics[width=0.31\textwidth]{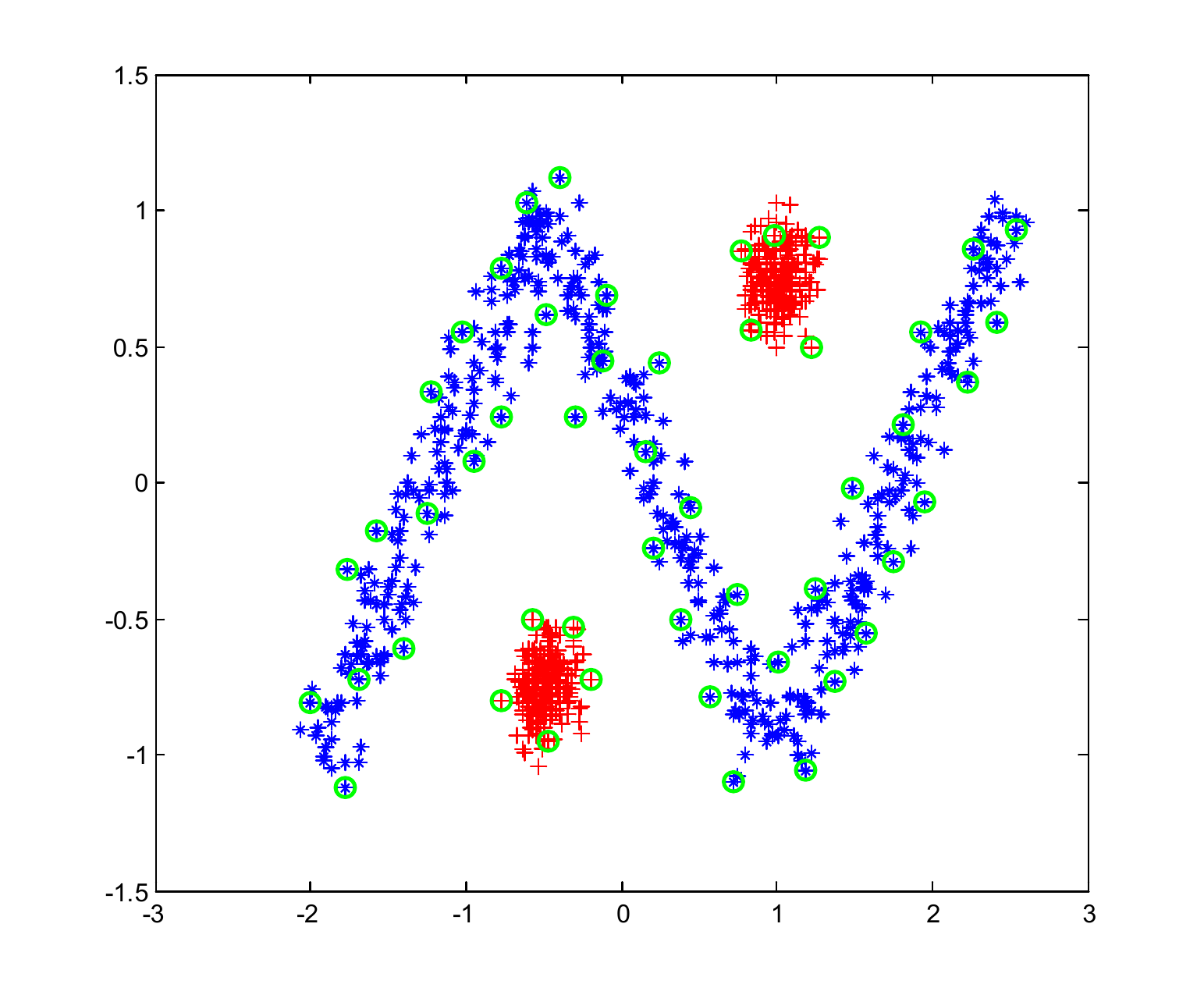}}
 \caption{Clustering performance of the proposed algorithm on three $2$-dimensional synthetic datasets. The green circles are the sampled data by the ICF. The clustering accuracy for each dataset is $100\%$. }
 \label{fig:Synthetic}
\end{figure}

\subsection{Real-world datasets experiments}
In order to evaluate the performance of the proposed algorithm, we compare it with several state-of-art kernel clustering algorithms on four real-world datasets in terms of clustering accuracy and time.
\subsubsection{Datasets}
We use two medium-size dataset and two large-size datasets to evaluate the performance of the algorithms. These datasets can be downloaded from LIBSVM website \footnote{https://www.csie.ntu.edu.tw/~cjlin/libsvmtools/datasets/}.
\begin{itemize}
\item \textbf{PenDigits:}The pen-based recognition of handwritten digits dataset contains $7,494$ training samples and $3,498$ test samples from $10$ classes. We combine them to form a dataset containing $10,992$ samples. Each sample is represented by a $16$-dimensional vector.
\item \textbf{Satimage:} This dataset contains 4,435 training points and 2,000 test points with $6$ classes. We combine them to form a dataset containing $10,992$ points. The dimension of the data is 36.
\item \textbf{Shuttle:} This is a dataset with $7$ classes containing $43,500$ data points. Each points has $9$ features. 
\item \textbf{Mnist:} This is a handwritten digits dataset containing $60,000$ data points. Each point is described by a vector of $780$ dimensions and assigned to one of $10$ classes, each class representing a digit.
\end{itemize}
\subsubsection{Baseline algorithms}\label{algorithms}
We compare our algorithm with the kernel $k$-means algorithms to verify that they achieve similar clustering accuracy. We also compare the proposed algorithm with improved kernel $k$-means algorithms, in which full kernel matrix need not be computed. The comparison algorithms are listed as follows:
\begin{itemize}
\item \textbf{Kernel:} The kernel $k$-means algorithm \cite{Scholkopf1998kkmeans} proposed by Sch\"{o}lkopf et al. This method requires to calculate the entire kernel matrix. The code has been included in the Matlab package.
\item \textbf{Kernel$+$Chol:} This algorithm calculates the entire kernel matrix first, then complete Cholesky factorization is used to decompose the entire kernel matrix. Finally, the $k$-means clustering algorithm is adopted on the rows of the decomposed matrix to obtain the clustering results.
\item \textbf{Approx:} The approximate kernel $k$-means algorithm \cite{Chitta2011approximate}, which employs a randomly selected subset of the data to compute the cluster centers.
\item \textbf{RFF:} The random fourier feature (RFF) kernel $k$-means clustering algorithm is proposed in \cite{Chitta2012efficient}. This algorithm applying RFF method to approximate the full kernel matrix, and then the $k$-means clustering is used to the points in the transformed space.
\item \textbf{Nystr\"{o}m:} The entire kernel matrix is approximated as $K_{\mathbb{MB}}K_{\mathbb{BB}}^\dag K_{\mathbb{MB}}^\top$ by Nystr\"{o}m method in this algorithm, where $\mathbb{B}$ is randomly sampled from $\mathbb{M}$. Then $k$-means clustering is applied on the rows of $K_{\mathbb{MB}}K_{\mathbb{BB}}^{-\frac{1}{2}}$.
\item \textbf{ICF:} The kernel $k$-means clustering using ICF is proposed by this paper.
\end{itemize}
\subsubsection{Parameters}
We use the Gaussian kernel function for all the algorithms. The kernel parameters $\sigma$ are set as $2^{-16}$, $2^{-3}$, $2$ and $2^{-6}$ for PenDigits, Satimage, Shuttle and Mnist datasets, respectively. The number of elements in $\mathbb{B}$ is denoted by ``subsetsize", which is varied from $25$ to $1000$. For approximate kernel $k$-means, RFF kernel $k$-means and Nystr\"{o}m kernel $k$-means algorithms, ``subsetsize" is the size of randomly selected subset, the number of Fourier components and the number of elements in $\mathbb{B}$, respectively. The error bound $\epsilon$ in ICF is set as $10^{-3}$ for all the datasets. The maximum number of iterations are set as $1000$ for $k$-means and kernel $k$-means. Each experimental result is the average of 10 independent experiments.
\subsubsection{Experimental results}

\begin{figure}[ht]
\centering
 \subfigure[PenDigits]{
    \label{fig:PenDigits_acc}
    \includegraphics[width=0.48\textwidth]{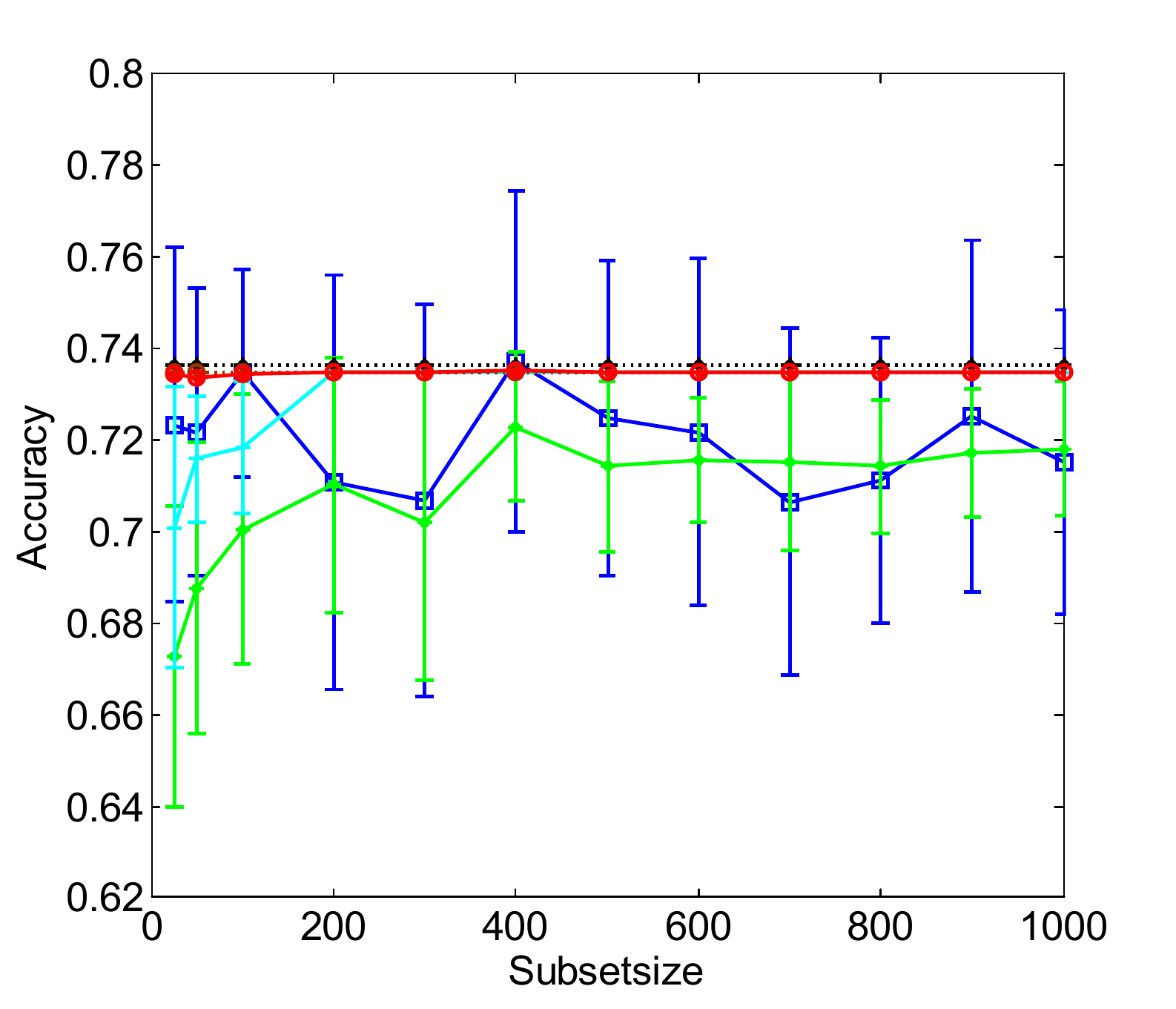}} 
 \subfigure[Satimage]{
    \label{fig:Satimage_acc}
    \includegraphics[width=0.48\textwidth]{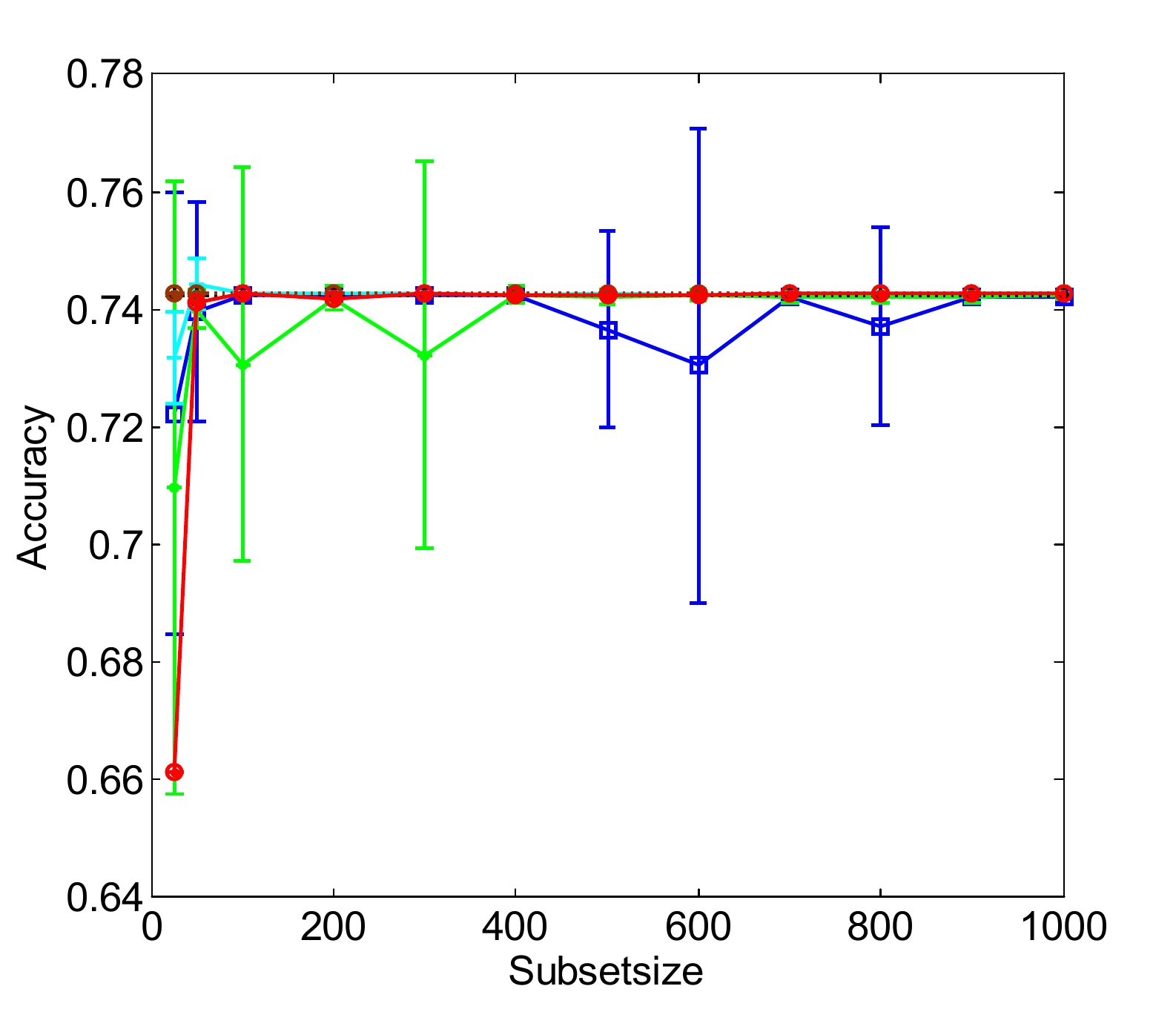}} \\
 \subfigure[Shuttle]{
    \label{shuttle_acc}
    \includegraphics[width=0.48\textwidth]{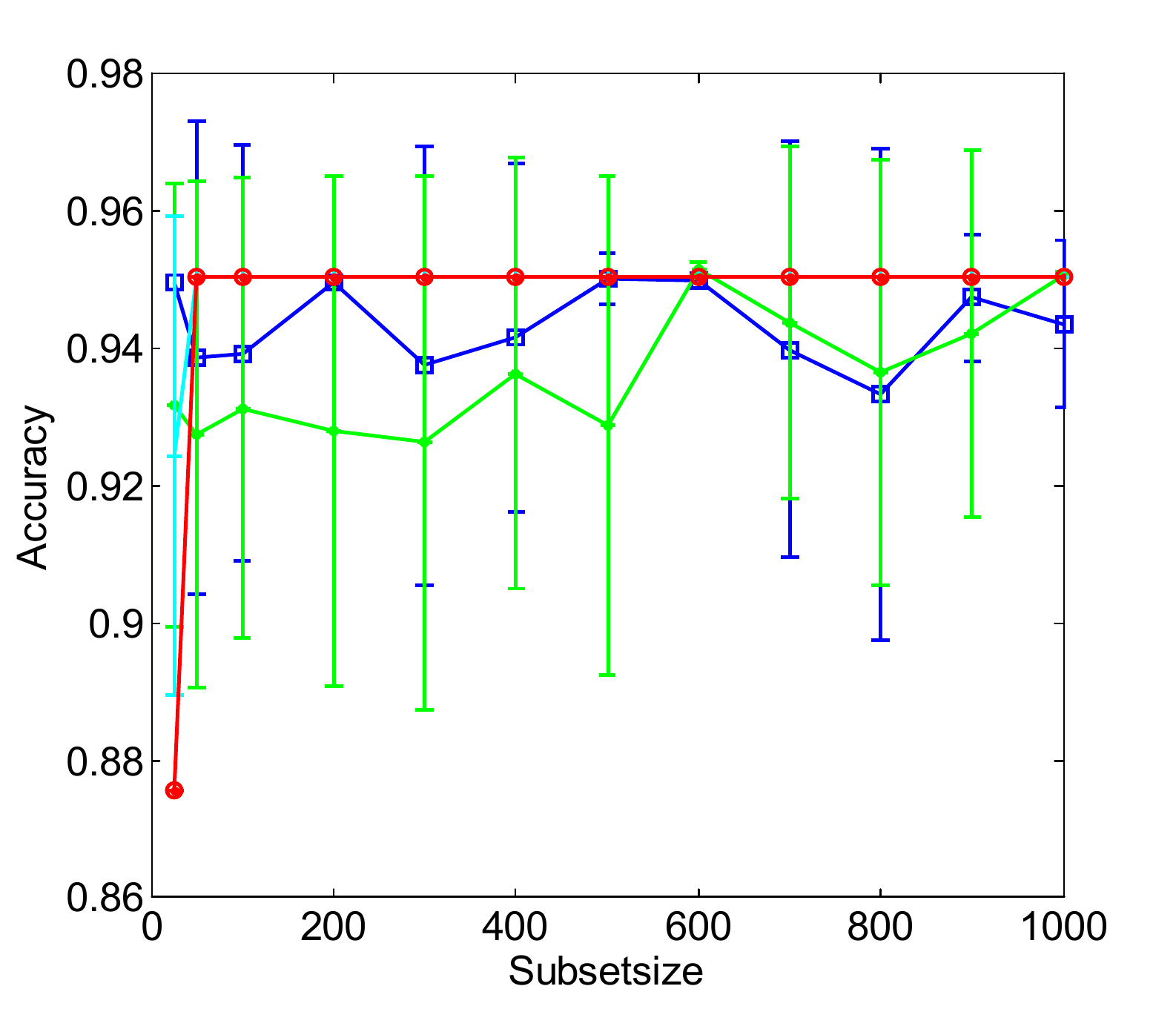}}
 \subfigure[Mnist]{
    \label{fig:mnist_acc}
    \includegraphics[width=0.48\textwidth]{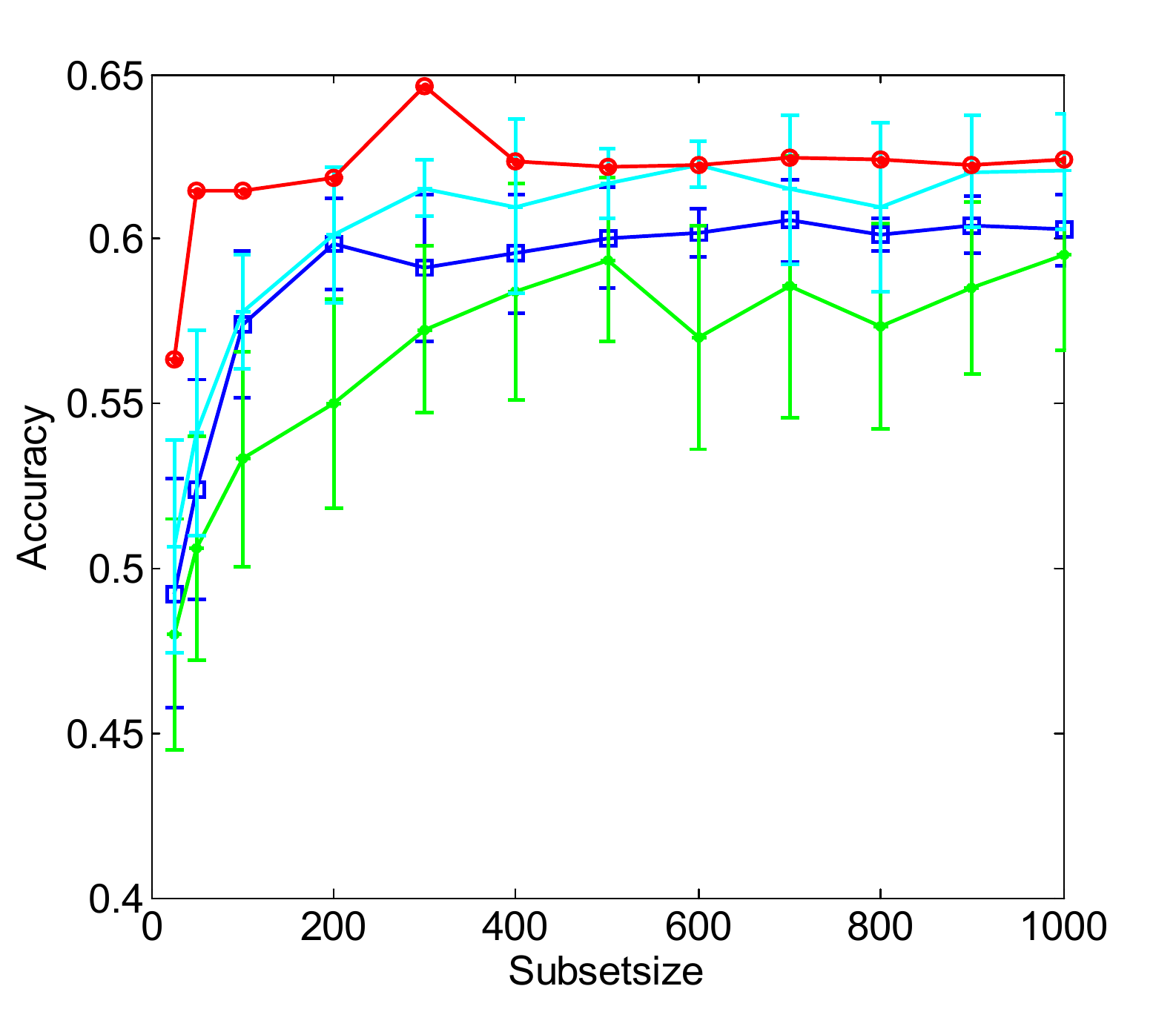}}\\
  \subfigure{
    \includegraphics[width=0.8\textwidth]{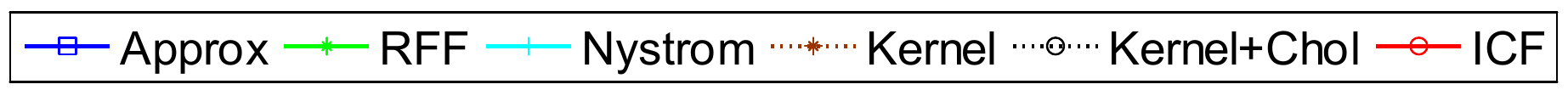}}
 \caption{Clustering accuracy comparison of algorithms on real-world datasets. }
 \label{fig:Real-world_acc}
\end{figure}
\begin{figure}[h]
\centering
 \subfigure[PenDigits]{
    \label{fig:Pendigits_time}
    \includegraphics[width=0.48\textwidth]{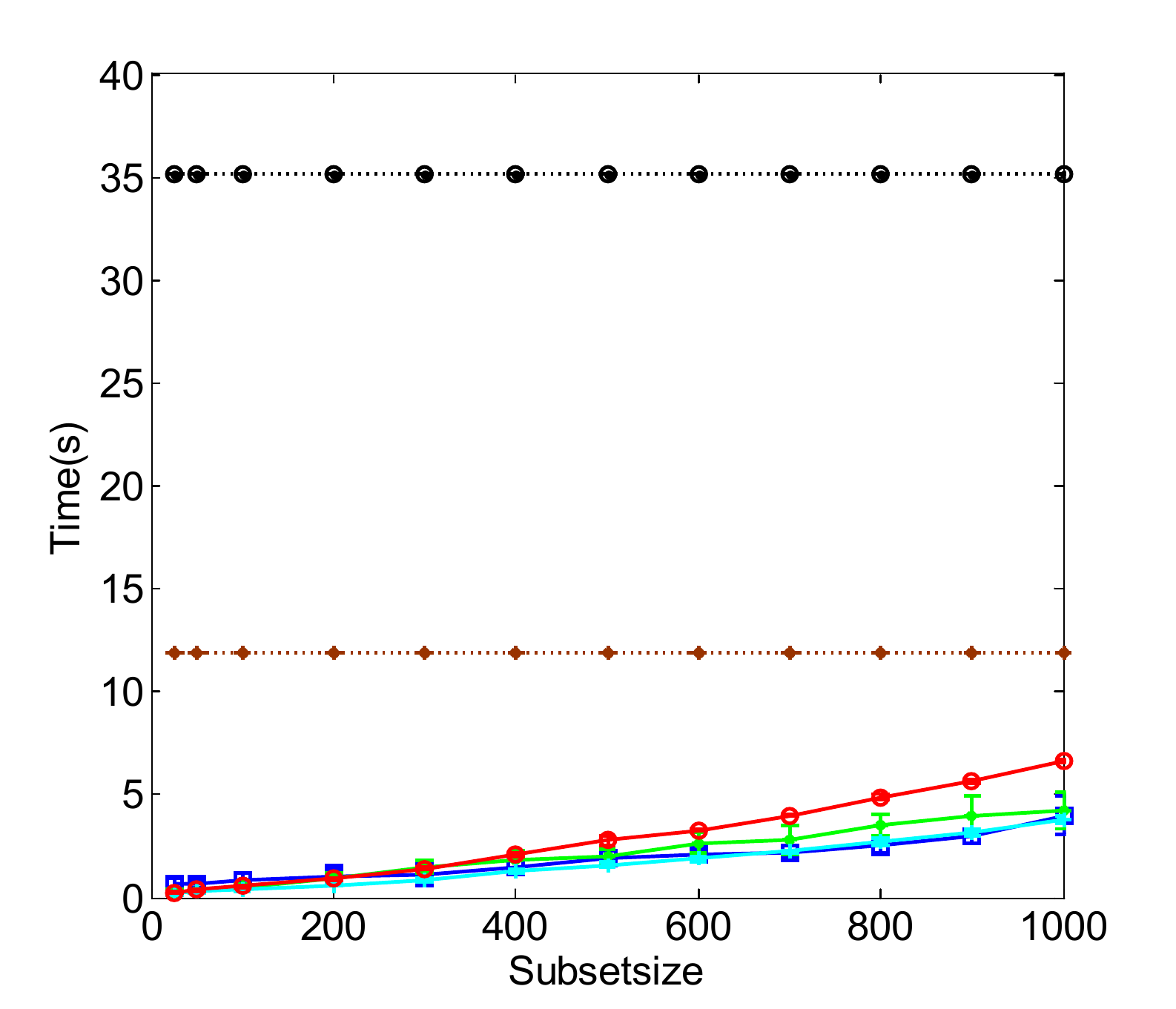}} 
 \subfigure[Satimage]{
    \label{fig:Satimage_time}
    \includegraphics[width=0.48\textwidth]{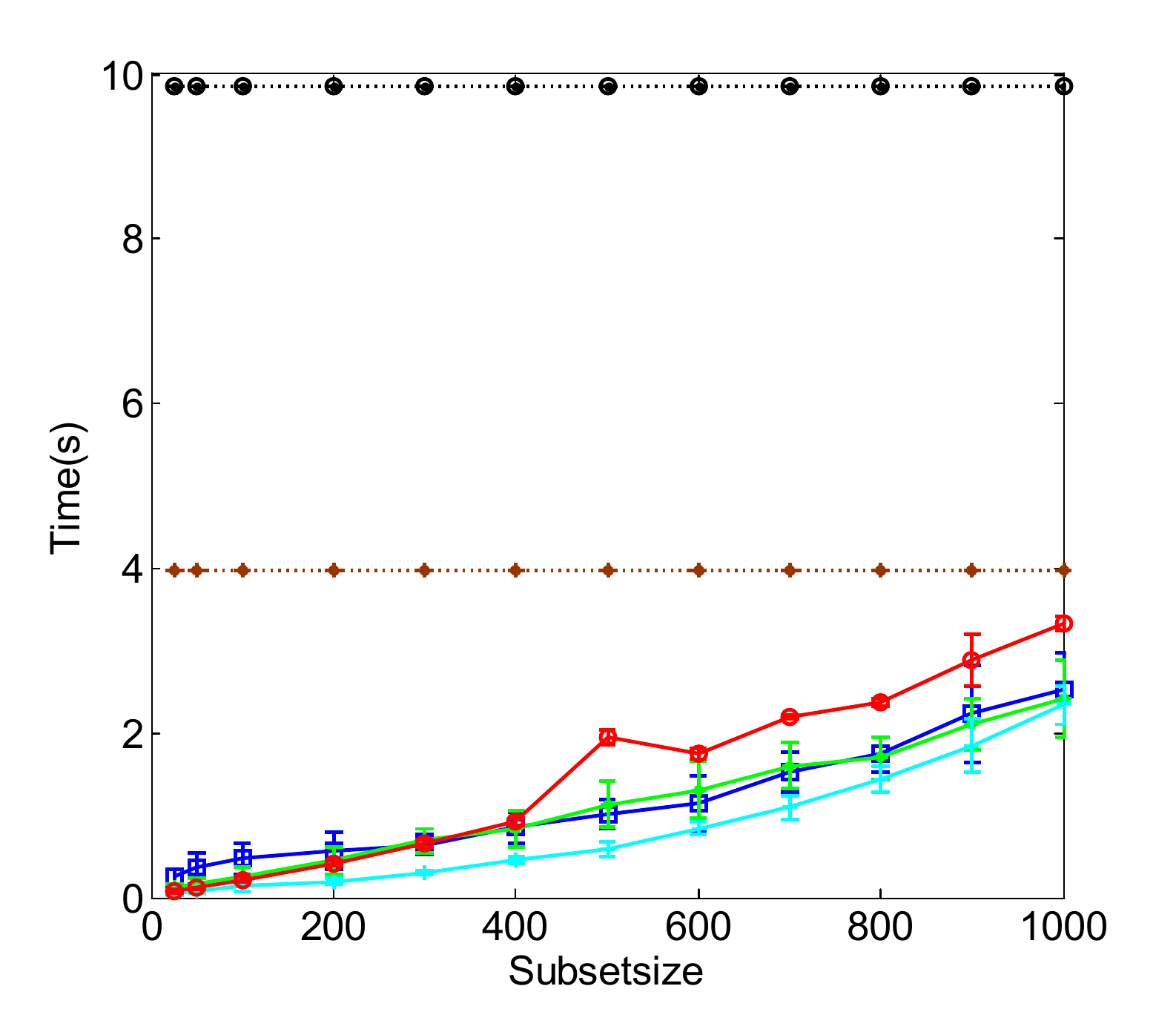}} \\
 \subfigure[Shuttle]{
    \label{shuttle_time}
    \includegraphics[width=0.48\textwidth]{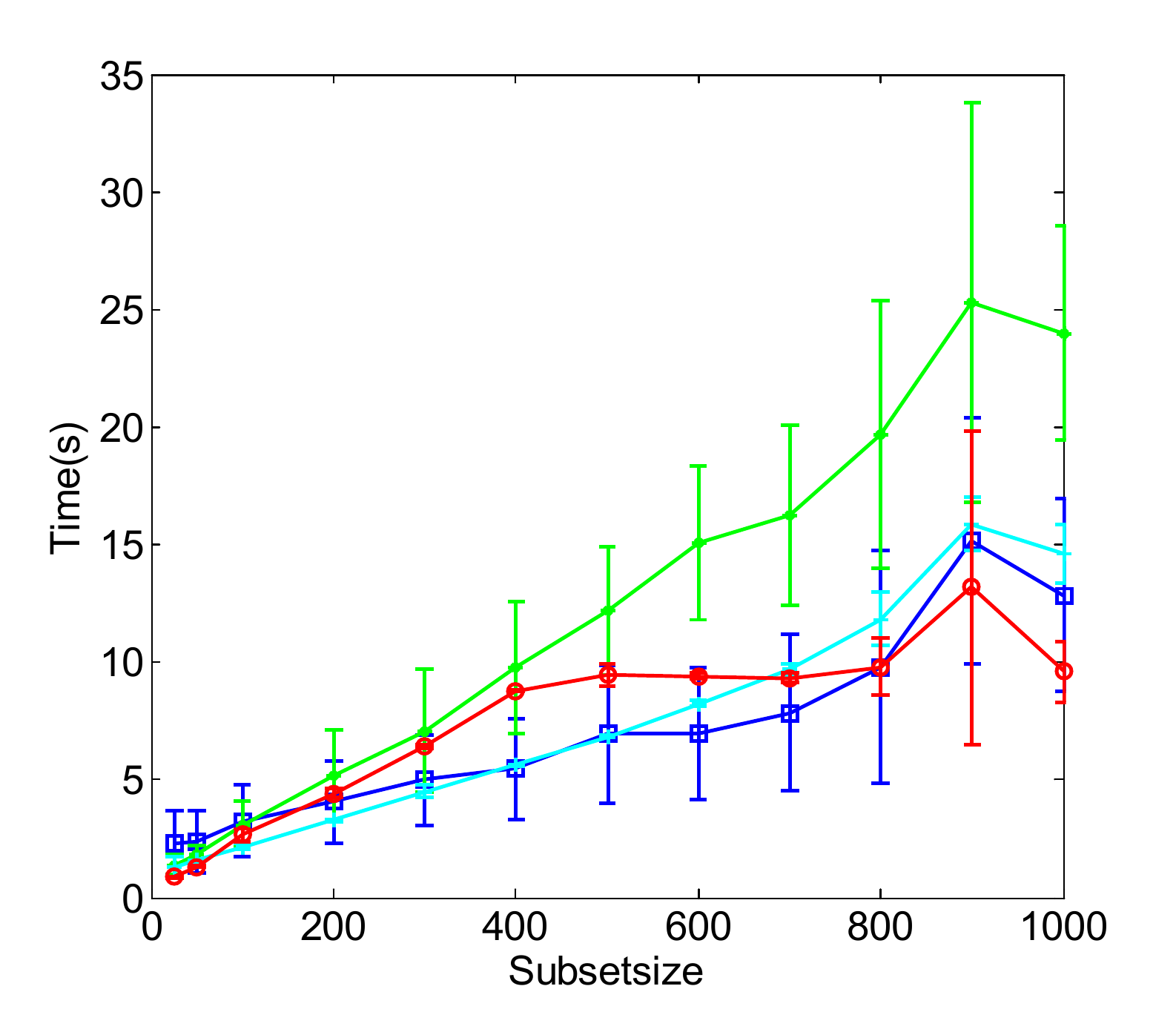}}
 \subfigure[Mnist]{
    \label{fig:mnist_time}
    \includegraphics[width=0.48\textwidth]{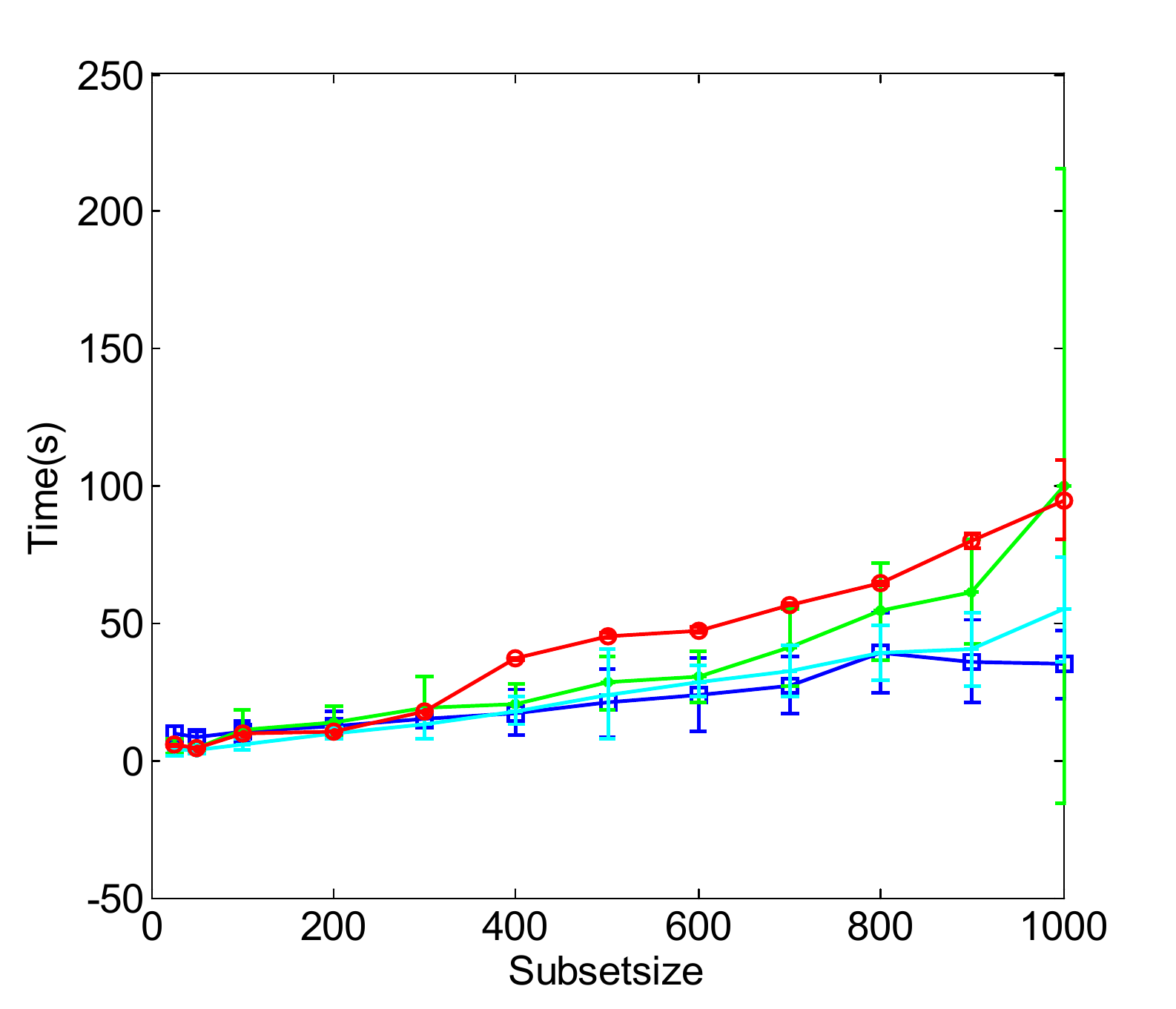}}\\
 \subfigure{
    \includegraphics[width=0.8\textwidth]{tuli_kkmeans}}
 \caption{Clustering time comparison of algorithms on real-world datasets. }
 \label{fig:Real-world_time}
\end{figure}

\begin{figure}[h]
\centering
 \subfigure[PenDigits, $s=25$]{
    \label{fig:Pendigits_time_opt}
    \includegraphics[width=0.48\textwidth]{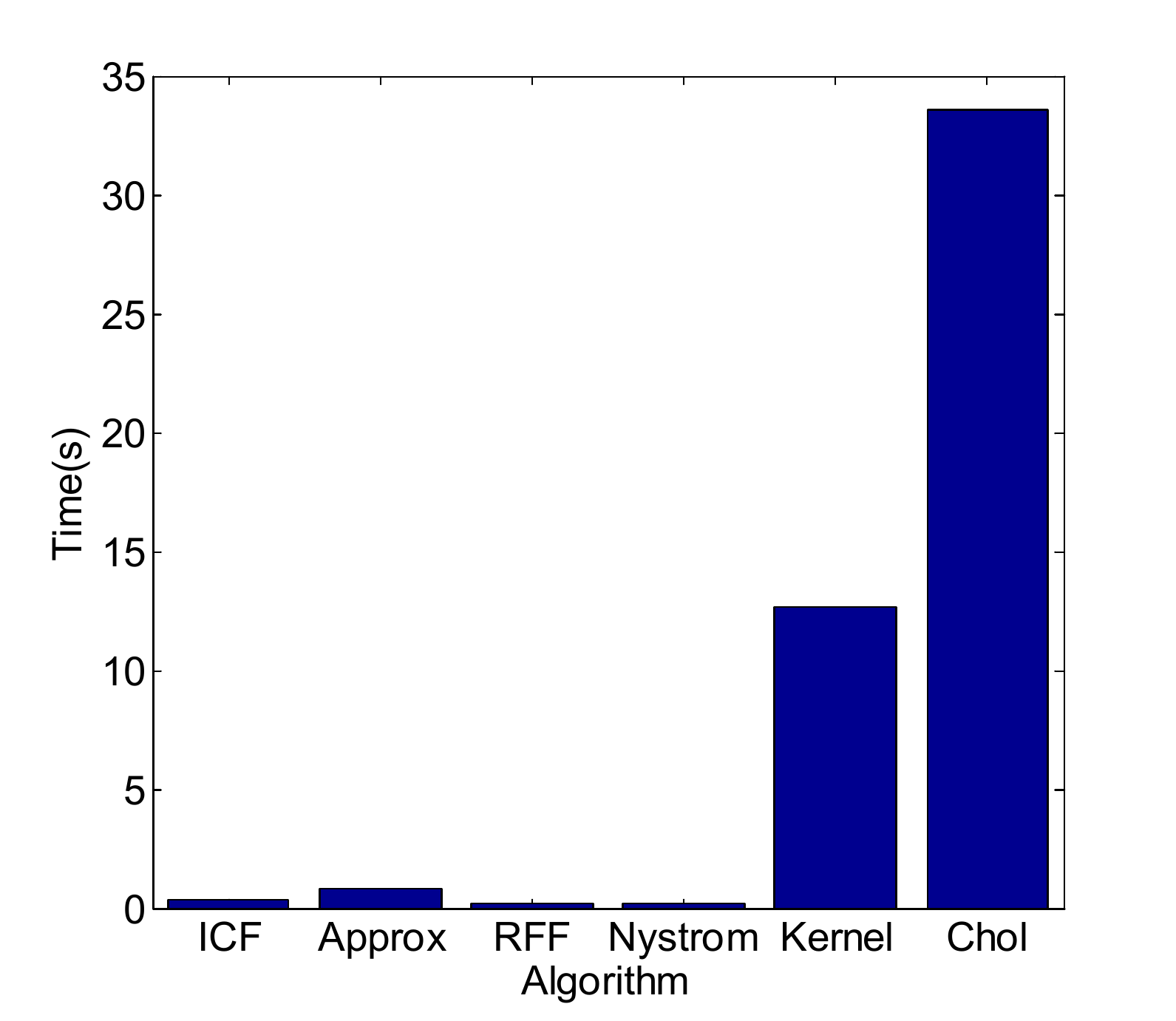}} 
 \subfigure[Satimage, $s=50$]{
    \label{fig:Satimage_time_opt}
    \includegraphics[width=0.48\textwidth]{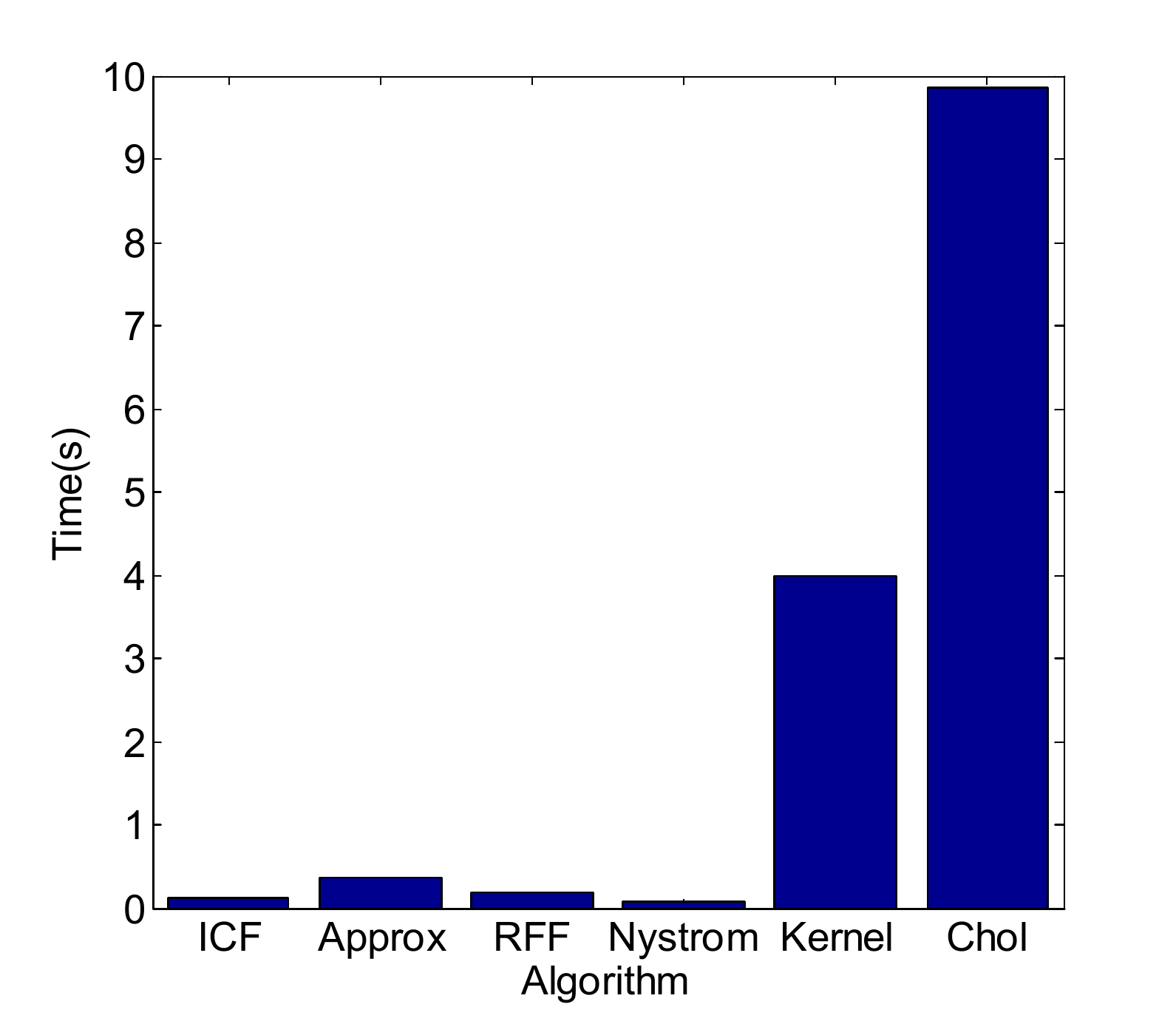}} \\
 \subfigure[Shuttle, $s=50$]{
    \label{shuttle_time}
    \includegraphics[width=0.48\textwidth]{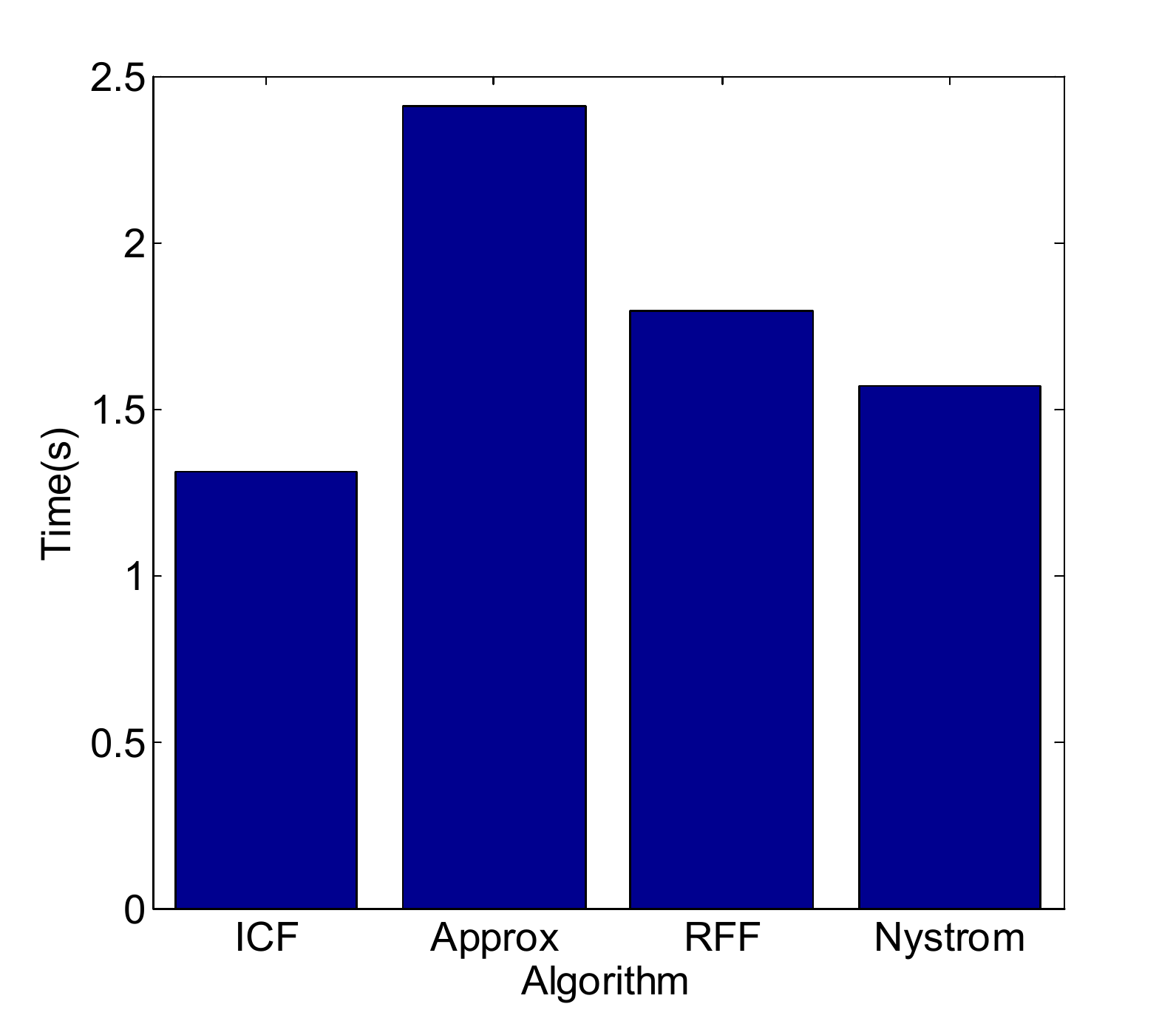}}
 \subfigure[Mnist, $s=50$]{
    \label{fig:mnist_time}
    \includegraphics[width=0.48\textwidth]{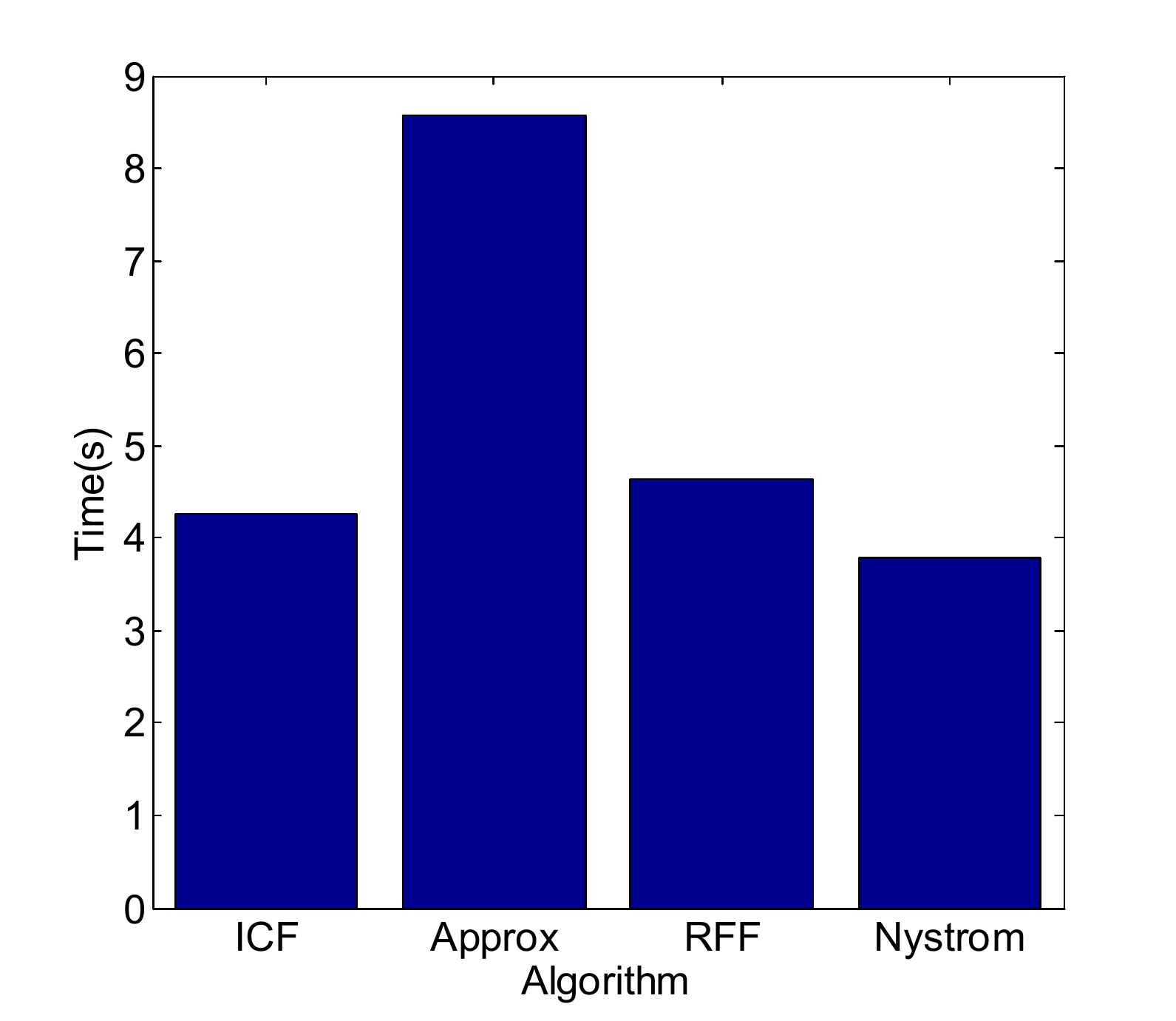}}
 \caption{When subsetsize sets as a small constant, clustering time comparison of algorithms on real-world datasets. $s$ indicates subsetsize. 'Chol' is the Kernel+Chol algorithm. $s=25, 50, 50$ and $50$ are sufficient to obtain satisfactory clustering accuracy for PenDigits, Satimage, Shuttle and Mnist datasets, respectively.}
 \label{fig:Real-world_time_opt}
\end{figure}

We compare all the algorithms mentioned in section \ref{algorithms} on PenDigits and Satimage datasets. However, for other two datasets, the full kernel matrix is infeasible, therefore we only compare the improved kernel $k$-means algorithms, in which full kernel matrix does not need to be computed. Figs. \ref{fig:Real-world_acc}- \ref{fig:Real-world_time_opt} show the clustering accuracy and running time of the algorithms on all datasets.
\begin{itemize}
\item Figs. \ref{fig:PenDigits_acc} and \ref{fig:Satimage_acc} illustrate that the accuracy of kernel $k$-means using complete Cholesky factorization algorithm (Kerenl+Chol) is similar with standard kernel $k$-means, which indicates the equivalence between \eqref{eq:kkmeans} and \eqref{eq:kkmeans2}.
\item The accuracy of the kernel $k$-means using ICF increases with the subsetsize increasing. That is because the larger the subsetsize is, the smaller the approximate error of ICF is.
\item For PenDigits and Satimage datasets, the proposed algorithm has the same accuracy as kernel $k$-means with full kernel matrix, when subsetsizes are larger than $25$ and $50$, respectively. However, Figs. \ref{fig:Pendigits_time_opt} and \ref{fig:Satimage_time_opt} illustrate that the clustering time of the new algorithm is great less than the kernel $k$-means clustering algorithms with full kernel matrix.
\item Fig. \ref{fig:Real-world_acc} indicates that the accuracy of ICF-based algorithm is better than the Nystr\"{o}m-based kernel $k$-means algorithm, RFF-based kernel $k$-means algorithm and approximate kernel $k$-means algorithm.
\item The new algorithm can achieve good clustering accuracy when subsetsize is larger than a very small constant. For PenDigits, Satimage, Shuttle and Mnist datasets, the subsetsizes are just set as 25, 50, 50 and 50 to achieve good clustering accuracy. Fig. \ref{fig:Real-world_time_opt} shows that the running time of the four algorithms (ICF, Approx, RFF and Nystrom) is similar, while the clustering accuracy of the new algorithm is better than the other three algorithms.
\item The variances of Nystr\"{o}m-based algorithm, RFF-based algorithm and Approximate kernel $k$-means are greater than the proposed algorithm. That is because these three methods are all based on the idea of randomly sampling.
\item In terms of running time, kernel $k$-means using complete Cholesky factorization algorithm is slower than standard kernel $k$-means due to complete Cholesky factorization consuming much time. However, our ICF-based algorithm greatly reduces running time, and faster than the standard kernel $k$-means. This verifies the effectiveness of our method.
\item Compared with three improved kernel $k$-means algorithms, the running time of the new algorithm increases slightly faster. However, when the subsetsize is small, the running time of these four algorithms is similar, while the clustering accuracy of our algorithm is better than that of the other three algorithms.
\end{itemize}
\section{Conclusion}

We have proposed a fast kernel $k$-means clustering algorithm, which uses incomplete Cholesky factorization and $k$-means clustering to obtain a good approximation of kernel $k$-means clustering. We have analyzed the convergence of ICF algorithm and shown that the ICF is exponentially convergent if the eigenvalues of the kernel matrix exponentially decrease. We also have bounded the approximate error between ICF-based kernel $k$-means algorithm and kernel $k$-means clustering algorithm, and shown that the approximate error decreases exponentially. The experimental results illustrate that the proposed method is able to yield similar clustering accuracy as the kernel $k$-means using entire kernel matrix, while the running time and the storage space are greatly reduced. In the future, we plan to investigate and research the minimum size of sampled subset required to yield similar accuracy as the kernel $k$-means with entire kernel matrix.

\section*{Acknowledgements}
This work is supported by the National Natural Science Foundation of China (NNSFC) [No. 61772020].
\section*{References}

\bibliography{fast_kkmeans}

\end{document}